\documentclass{article}

\PassOptionsToPackage{numbers, compress}{natbib}

\usepackage[preprint]{neurips_2026}

\usepackage[utf8]{inputenc} 
\usepackage[T1]{fontenc}    
\usepackage{hyperref}       
\usepackage{url}            
\usepackage{booktabs}       
\usepackage{amsfonts}       
\usepackage{amsmath,amssymb,mathtools,bm}
\usepackage{amsthm}
\hypersetup{hidelinks}
\usepackage{nicefrac}       
\usepackage{microtype}      
\usepackage{xcolor}         
\usepackage{graphicx}
\usepackage{algorithm,algpseudocode}
\usepackage{multirow,makecell}
\usepackage{caption}
\usepackage{subcaption}
\usepackage{wrapfig}
\usepackage{tikz}
\usepackage{fontawesome5}
\usepackage{enumitem}
\usetikzlibrary{positioning,arrows.meta,calc,fit}

\newtheorem{theorem}{Theorem}
\newtheorem{proposition}{Proposition}
\newtheorem{definition}{Definition}
\newtheorem{assumption}{Assumption}
\newtheorem{corollary}{Corollary}

\newcommand{\softmax}{\mathrm{softmax}}

\newcommand{\UAT}{\text{UAT}}

\newcommand{\R}{\mathbb{R}}

\newcommand{\diag}{\mathrm{diag}}
\newcommand{\LN}{\mathrm{LN}}

\newcommand{\TF}{\text{TF}}

\newcommand{\obs}{\mathrm{obs}}

\newcommand{\rank}{\operatorname{rank}}
\newcommand{\col}{\operatorname{col}}
\newcommand{\row}{\operatorname{row}}

\title{Breaking Feedback-Blindness: Utility-Augmented Transformer for Sequential Decision Making}

\author{
Yuyang Shen$^{1}$ \quad Shan Dai$^{2, 1, *}$ \quad Daimin Chen$^{3}$ \\
$^{1}$The Chinese University of Hong Kong, Shenzhen \\
$^{2}$Shenzhen Research Institute of Big Data \\
$^{3}$University of International Business and Economics \\
\texttt{yuyangshen@link.cuhk.edu.cn} \quad
\texttt{daishan@cuhk.edu.cn} \quad
\texttt{{202520210806}@uibe.edu.cn}
}

\begin{document}
\maketitle

\begingroup
\renewcommand{\thefootnote}{*}
\footnotetext{Corresponding author.}
\endgroup

\begin{abstract}

Sequential decision making in non-stationary and partially observable environments requires rapid adaptation to latent regime changes. However, existing Transformer decision models face a structural bottleneck in the retrieval mechanism: even when reward is used for training or exposed as an input token, attention retrieval remains primarily driven by observation-derived similarity. We formalize this limitation as \textbf{feedback-blind retrieval}, and formally show that, on feedback-informative tasks, observation-equivalent histories with different action--reward outcomes cannot be distinguished by any observation-only attention, resulting in suboptimal choice. To address this mismatch, we propose the \textbf{Utility-Augmented Transformer} (UAT), a new feedback-conditioned retrieval attention architecture in which a compact utility state modulates the query, key, and value projections, allowing action--reward history to directly alter context retrieval during the forward pass. UAT also enjoys an exact zero-gate degradation property that recovers the Vanilla Transformer when feedback is uninformative. Under finite-horizon compactness and Lipschitz assumptions, we prove that UAT strictly enlarges the observation-only Transformer class and can uniformly approximate feedback-dependent decision maps. Across four non-stationary benchmarks---synthetic navigation with hidden goal shifts, non-stationary sepsis treatment, cross-market portfolio allocation, and delayed-feedback recommendation, UAT consistently improves performance over observation-only, test-time adaptation, and input-level feedback baselines, with particularly large gains in noisier regimes that require stronger adaptation.



\end{abstract}

\section{Introduction}
\label{sec:intro}

Sequential decision making is central to reinforcement learning and to a broad range of applications, including clinical treatment planning~\citep{komorowski2018ai}, financial portfolio management~\citep{ye2020}, recommender systems~\citep{afsar2022reinforcement}, robotics~\citep{levine2016end}, and autonomous control~\citep{kiran2021deep}. In recent years, Transformer architectures~\citep{vaswani2017attention} have emerged as a compelling backbone for decision models because of their strong sequence modeling capacity and their ability to aggregate long-range contextual evidence. Decision Transformer (DT)~\citep{chen2021decision} and its extensions~\citep{zheng2022odt,wu2023edt,schmied2024radt} have shown that many problems can be profitably reframed as sequence modeling, stimulating a large literature on Transformer-based reinforcement learning. 
	
Despite these advances, non-stationary and partially observable environments remain difficult~\citep{cheung2020reinforcement}. The challenge is not only that the distribution shifts, but also that the latent regime governing action effectiveness may not be identifiable from the observation sequence alone~\citep{wilder2019melding, vlastelica2020differentiation}. In many realistic settings, two episodes can exhibit nearly identical observations while requiring different decisions because prior actions produced different reward feedback~\citep{elmachtoub2022spo, donti2017task, mandi2024dfl}. 
	
This difficulty exposes a structural mismatch~\citep{kaelbling1998pomdp} in many decision architectures. Existing models often use reward either as a training signal~\citep{schulman2017ppo,kostrikov2022offline} or as an input token~\citep{chen2021decision, schmied2024radt}, but the retrieval rule that governs how the model accesses historical context is still predominantly shaped by observation-derived attention~\citep{parisotto2020stabilizing, kang2018sasrec}. When the frozen forward computation induces identical context retrieval for observation-matched histories, downstream supervision cannot distinguish them within the same forward pass~\citep{rumelhart1986learning, bengio2013representation}. This is the core phenomenon we call \emph{feedback-blind retrieval}. To make this precise, we define \emph{feedback-informative decision processes}: tasks where there exist two histories with identical observation sequences but different past action--reward outcomes, and these histories induce different optimal decisions. The definition isolates the regime-identification setting~\citep{zintgraf2020varibad} in which feedback is not merely useful but structurally necessary. We then define \emph{feedback-blind retrieval} as invariance of the retrieval state to past action--reward feedback when observations are held fixed. 

This structural mismatch has motivated a growing body of work aimed at incorporating reward and action information into decisions~\citep{arulkumaran2017brief}. Early efforts, such as the Decision Transformer (DT)~\citep{chen2021decision} and its variants~\citep{zheng2022odt,wu2023edt,schmied2024radt}, introduce return-to-go actions and rewards as additional input tokens. These tokens can influence attention scores, but the projection operators defining the retrieval kernel remain fixed and are not themselves conditioned on feedback, so observation and feedback interactions must share the same bilinear retrieval pathway. Similarly, gradient-based RL training~\citep{sutton2018reinforcement} and test-time adaptation methods~\citep{wang2021tent} can update network parameters to reflect new feedback, but because these updates occur outside the forward pass, the model's retrieval state during inference remains blind to the within-episode action--reward history~\citep{wang2022cotta}. In other words, no matter how sophisticated the training signal or the parameter update rule is, a frozen forward pass that retrieves context based solely on observation features cannot distinguish histories with identical observations but divergent feedback outcomes~\citep{cover2006elements}.

We therefore identify retrieval—not merely representation or training—as the fundamental bottleneck. To break this bottleneck, what is needed is not another auxiliary loss~\citep{donti2017task,mandi2024dfl} or a richer input tokenization~\citep{chen2021decision}, but a structural mechanism that allows the attention kernel itself to be directly conditioned on within-episode action--reward history. Motivated by this distinction, we propose the Utility-Augmented Transformer (UAT). UAT introduces a lightweight utility encoder that maps shifted action--reward pairs to a compact utility state. This state conditions the query, key, and value projections through additive and multiplicative modulation pathways~\citep{ha2017hypernetworks,perez2018film}, thereby allowing feedback to directly reshape the attention kernel during the forward pass. 


\textbf{Contributions.}
The major contributions of this work are as follows.
\begin{itemize}[leftmargin=1.5em,itemsep=0.2em]

    \item \textbf{Feedback-blind retrieval formulation.}
		We identify a class of \emph{feedback-informative} decision processes in which optimal actions cannot be identified from observations alone. On the model side, we define \emph{feedback-blind retrieval} and prove that observation-only attention cannot distinguish histories with identical observations but different action--reward outcomes.
    
    \item \textbf{Utility-Augmented Transformer.} We propose UAT, a new feedback-conditioned retrieval architecture that promotes the attention query, key, value projections to utility-affine functions of an action--reward state, while preserving compatibility with Vanilla Transformers when the utility gates are zero.
    
   	\item \textbf{Expressiveness and approximation theory.} Theoretically, we prove that UAT strictly generalizes the observation-only Transformer function class and, under explicit assumptions, can uniformly approximate any continuous feedback-dependent Lipschitz decision map.
        
		
        
    \item \textbf{Empirical evaluation.} We perform a comprehensive empirical study on four heterogeneous benchmarks: UAT consistently outperforms baselines without incurring substantial computational overhead. Boundary analysis shows degradation as feedback weakens. Ablations isolate the role of each component; and mechanism analysis reveals interpretable modulation functional patterns.
\end{itemize}


\section{Related Work}


\textbf{Observation-Only Transformers.} 
Observation-only Transformer refers to Transformers that preserve the core attention-retrieval structure, retrieving context through observation-derived queries and keys. A Vanilla Transformer~\citep{vaswani2017attention} with a downstream action head is the simplest backbone of this family~\citep{yun2020transformers,dai2019transformerxl,li2023transformer_rl}.  In sequential decision making, we use its causal form~\citep{melnychuk2022causal}: the architecture is unchanged, but attention is masked so that decisions depend only on existing information, preventing leakage from future information. The shared architecture is as follows. At step $t$, each observation feature vector $o_j\in\mathcal{O}$ for $j\le t$ is embedded as:
\begin{equation}
  x_j = \mathrm{LN}\!\bigl(W^E\,o_j + p_j\bigr),
  \label{eq:embed}
\end{equation}
where $W^E\in\mathbb{R}^{d\times d_o}$ is a learnable matrix, $p_j\in\mathbb{R}^d$ is the positional encoding, $\mathrm{LN}$ denotes pre-normalization. For the current embedded vector $x_t$ and each historical $x_j$ ($j\le t$), attention computes:
\begin{align}
  q_t &= W_Q x_t,\quad k_j = W_K x_j,\quad v_j = W_V x_j, \label{eq:qkv}\\
  h_t^{(l)} &= \sum_{j=1}^{t}\softmax_j\!\Bigl(\tfrac{\langle q_t,k_j\rangle}{\sqrt{d_m}}\Bigr)\, v_j, \label{eq:attn-inner}
\end{align}
where $W_Q, W_K, W_V\in\mathbb{R}^{d_m\times d}$ are called projection matrices and $\langle\cdot,\cdot\rangle$ is the Euclidean inner product. Multi-head attention~\citep{vaswani2017attention} runs this in parallel across heads with independent projections, and each block follows it with a residual connection~\citep{he2016deep} and a position-wise FFN~\citep{rumelhart1986learning} . $h_t^{(l)} \in \mathbb{R}^{d}$ is the output vector at layer $l$, which is passed to the next layer as input. After $L$ layers, the final state $h_t^{(L)}\in\mathbb{R}^{d}$ is passed to the decision head and the model is trained with cross-entropy~\citep{shannon1948mathematical}.

GTrXL~\citep{parisotto2020stabilizing} makes it more stable: it adds gated residual connections and a memory mechanism so that the model can preserve information over long horizons and train more reliably. However, GTrXL still receives the same observation-token sequence as Vanilla-TF. Its attention scores are computed from observation features only. Although rewards affect the parameters through the training loss, they do not form a within-episode signal that changes how the frozen model retrieves past context~\citep{bengio2013representation}.

\textbf{Input-Level Feedback Injection Models.} Decision Transformer (DT)~\citep{chen2021decision} incorporates feedback at the input level by casting control as conditional sequence modeling. At each step, it forms an interleaved sequence of return-to-go, observation, and action tokens, $(\hat{R}_t, o_t, a_t)_{t=1}^T$, and processes the resulting trajectory with a standard causal Transformer~\citep{melnychuk2022causal}. Later variants improve this paradigm along different aspects: ODT~\citep{zheng2022odt} introduces online finetuning with stochastic policies; EDT~\citep{wu2023edt} adapts the effective history length at inference to support trajectory stitching; and RA-DT~\citep{schmied2024radt} augments the context with retrieved sub-trajectories from an external memory. These methods enrich the context procedure, but largely retain the standard attention kernel, so the role of feedback in their retrieval remains structurally limited.

\textbf{Adaptation or Conditioning Outside the Retrieval Kernel.} Another related line changes model behavior without modifying the attention retrieval rule. Test-time adaptation methods such as TENT~\citep{wang2021tent} and CoTTA~\citep{wang2022cotta} update parameters or normalization statistics during deployment to handle distribution shift, but their adaptation signal is usually an unsupervised prediction objective, and the attention kernel remains fixed. Feature-conditioning methods such as FiLM~\citep{perez2018film}, AdaIN~\citep{huang2017adain}, and SPADE~\citep{park2019spade} inject external signals by modulating intermediate activations. They demonstrate the value of conditional computation, but operate on hidden features after retrieval.


\section{Problem Setup and Structural Mismatch}
\label{sec:mismatch}

\begin{definition}[Non-stationary POMDP]\label{def:pomdp}
A non-stationary POMDP is a tuple $\mathcal{M}=(\mathcal{S},\mathcal{A},\mathcal{O},\{P_t\}_{t\ge1},\{R_t\}_{t\ge1},\Omega_t,\gamma)$, where $P_t$ and $R_t$ are transitions and reward functions. At each step $t$, with latent state $s_t \in \mathcal{S}$ (unobserved), the agent observes $o_t \sim \Omega_t(\cdot \mid s_t)$, takes action $a_t \in \mathcal{A}$, and receives reward $r_t = R_t(s_t,a_t)$. Decisions are made from the observable history $\mathcal{H}_t$:
\begin{equation}
  \mathcal{H}_t
  = \bigl\{(o_1, a_1, r_1), \ldots, (o_{t-1}, a_{t-1}, r_{t-1}), o_t\bigr\}.
  \label{eq:history}
\end{equation}
A history-dependent policy $\pi: \mathcal{H}_t \to \Delta(\mathcal{A})$ maps histories to action distributions. Under non-stationarity, the optimal policy is itself time-varying:
\begin{equation}
  \pi_t^*(\cdot \mid \mathcal{H}_t)
  \in
  \arg\max_{\pi}\;
  \mathbb{E}_M^\pi\!\left[
    \sum_{k=t}^{T} \gamma^{k-t} r_k
    \;\middle|\;
    \mathcal{H}_t
  \right]
  \label{eq:ns-opt}
\end{equation}
\end{definition}

Definition~\ref{def:pomdp}  highlights that under non-stationarity, the relevant uncertainty is often \emph{regime uncertainty}: the same observation may have different control implications depending on how the environment has responded to previous actions~\citep{cheung2020reinforcement}. This is the setting in which action--reward feedback becomes decision-critical. We next formalize the task-side property that motivates this work.

\subsection{Feedback-Informative Tasks and Feedback-Blind Retrieval}\label{sec:fed-inform}
	

\begin{definition}[Feedback-Informative Decision Process]
\label{def:feedback-informative}
A non-stationary POMDP $\mathcal{M}$ is \emph{feedback-informative at time $t$} if there exist two  histories $\mathcal{H}_t$ and $\mathcal{H}'_t$ as defined in \eqref{eq:history}
that share the same observation sequence $\{o_j\}_{j\le t}$ but differ in past action--reward outcomes, such that
\begin{equation}
\pi_t^*(\cdot \mid \mathcal{H}_t)
\neq
\pi_t^*(\cdot \mid \mathcal{H}'_t).
\end{equation}
\end{definition}
Definition~\ref{def:feedback-informative} does not merely restate partial observability. It isolates a stronger phenomenon: \emph{feedback-relevant aliasing}: even after conditioning on entire observation sequence, optimal actions may remain unidentified unless the model knows how environment responded to previous actions. This is the key property of regime-shifting tasks, where observation-equivalent histories must be treated differently.

\begin{definition}[Feedback-blind retrieval]\label{def:feedback-blind}
		For a sequential decision model with retrieval representation $R_t(\mathcal{H}_t)$ at time $t$, we say that the model admits \emph{feedback-blind retrieval} if for any two histories $\mathcal{H}_t$ and $\mathcal{H}'_t$ in \eqref{eq:history} that share the same observation sequence $\{o_j\}_{j\le t}$,
		\begin{equation}
		R_t(\mathcal{H}_t)=R_t(\mathcal{H}'_t).
		\end{equation}
\end{definition}
Definition~\ref{def:feedback-blind} is a structural property of the \emph{forward-pass retrieval mechanism}, not a statement about how the parameters were trained. A model may be trained with reward-based objectives and yet remain feedback-blind at inference if its frozen retrieval state is still solely a function of observations.


\subsection{Why Observation-Driven Retrieval Fails}

	\begin{proposition}[Observation-only Transformers are feedback-blind]
		\label{prop:obs_blind}
		Consider any causal Transformer policy whose input embeddings, attention blocks, and task head depend only on the observation sequence $\{o_j\}_{j\le t}$ at a fixed parameter value $\theta$. Then for any two histories $\mathcal{H}_t,\mathcal{H}_t'$ with identical observation sequence $\{o_j\}_{j\le t}$, the final retrieval representation satisfies
		$
		R_t(\mathcal{H}_t)=R_t(\mathcal{H}_t').
		$
		Hence the model admits feedback-blind retrieval.
	\end{proposition}

		Proposition~\ref{prop:obs_blind} does not deny that feedback can be used through recurrent hidden states, belief-state updates, or test-time parameter adaptation. It says that a frozen observation-only attention kernel cannot distinguish observation-equivalent but feedback-distinct histories in the same forward pass. A second limitation concerns input-level feedback injection.

\begin{proposition}[Input-level injection retains a rank-limited, entangled retrieval kernel]\label{prop:adt_limit}
Consider a concatenative feedback model with token embedding $x_j = W_e [o_j; f_j]$, where $f_j := [a_{j-1}; r_{j-1}]$ and $W_e = [W_e^{(o)}, W_e^{(f)}]$ is partitioned column-wise. Then the attention logit between query position $t$ and key position $j$ decomposes as
\begin{equation}
\ell_{tj} = \underbrace{o_t^\top M^{(oo)} o_j}_{\text{obs--obs}} + \underbrace{o_t^\top M^{(of)} f_j}_{\text{obs--fb}} + \underbrace{f_t^\top M^{(fo)} o_j}_{\text{fb--obs}} + \underbrace{f_t^\top M^{(ff)} f_j}_{\text{fb--fb}},
\label{eq:input-level-decomp}
\end{equation}
with $M^{(oo)} := (W_e^{(o)})^{\top} W_Q^{\top} W_K W_e^{(o)}$, $M^{(of)} := (W_e^{(o)})^{\top} W_Q^{\top} W_K W_e^{(f)}$, $M^{(fo)} := (W_e^{(f)})^{\top} W_Q^{\top} W_K W_e^{(o)}$, and $M^{(ff)} := (W_e^{(f)})^{\top} W_Q^{\top} W_K W_e^{(f)}$. Every interaction matrix factors through the shared product $W_Q^{\top} W_K \in \R^{d \times d}$ of rank at most $d_m$; in particular $\rank(M^{(\cdot\cdot)}) \le d_m$, $\col(M^{(o\cdot)}) \subseteq \col\big((W_e^{(o)})^{\top} W_Q^{\top}\big)$, and $\row(M^{(\cdot o)}) \subseteq \row\big(W_K W_e^{(o)}\big)$. Hence all four interaction modes are routed through the same pair of projection maps.
\end{proposition}

Proposition~\ref{prop:adt_limit} is a per-layer statement about the bilinear structure of the attention kernel: input-level injection changes what is stored in the token features, but each layer's retrieval score remains a single rank-$d_m$ bilinear form in which the observation--observation, observation--feedback, and feedback--feedback modes cannot be assigned independent pathways. Depth can mix the channels indirectly through value updates, which explains why input-level models perform well on simpler tasks in our experiments; the bottleneck is that no layer can directly re-route retrieval based on feedback without entangling it with observation similarity. Conditioning the projections on feedback, as UAT does, provides this disentangled, direct pathway. The proof details of Proposition~\ref{prop:obs_blind} and Proposition~\ref{prop:adt_limit} are in Appendix~\ref{prof:prop2}

\section{Utility-Augmented Transformer}
\label{sec:uat}


\subsection{Utility Encoder}
\label{sec:utility-enc}

We first embed the most recent action--reward pair $(a_{j-1},r_{j-1})$ at observation token $j$ into $u_j$. The resulting $u_j$ serves as a compact feedback state that directly conditions downstream decisions:
\begin{equation}
  u_j = E\bigl(f_{\mathrm{a}}(a_{j-1}),\; r_{j-1}\bigr) \in \mathbb{R}^{d_u},
  \label{eq:utility-enc}
\end{equation}
where $f_{\mathrm{a}}(a_{j-1})=W_{\mathrm{a}}\,a_{j-1}$ is a bias-free linear function with learnable parameter $W_{\mathrm{a}}\in\mathbb{R}^{d_{\mathrm{a}}\times |\mathcal{A}|}$. $E$ is a lightweight neural network named \emph{utility encoder}, and $d_u$ denotes its output dimension. In practice, we adopt the instantiation: we set $E$ as a bias-free two-layer MLP~\citep{hornik1989multilayer} with ReLU~\citep{nair2010relu} activation:
\begin{equation}
  u_j = \mathrm{MLP}\bigl(\bigl[\,f_{\mathrm{a}}(a_{j-1});\; r_{j-1}\,\bigr]\bigr).
  \label{eq:ue-simple}
\end{equation}

The \emph{token-level gating vector} $g_j^{(l,h)}$ and \emph{regime-level gating vector} $\rho_t^{(h)}$ are then derived from $u_j$:
\begin{align}
  g_j^{(l,h)} &= g_u^{(l,h)}(u_j) \in [-1,1]^{d_m}, \qquad  \rho_t^{(h)} = f_u^{(h)}(\bar{u}_t) \in [-1,1]^{d_m},
  \label{gate} \\
  \bar{u}_t &= \lambda \bar{u}_{t-1} + (1-\lambda) u_t,
  \label{eq:discount-u}
\end{align}
with initialization $\bar{u}_0=\mathbf{0}$ and $(a_0,r_0)=(\mathbf{0},0)$. Here $l$ and $h$ index layer and head. The readouts $g_u^{(l,h)}$ and $f_u^{(h)}$ are bounded Lipschitz maps~\citep{rudin1976principles} from $\mathbb{R}^{d_u}$ to $[-1,1]^{d_m}$, and $\lambda\in(0,1)$ sets the exponential memory. The local gate $g_j^{(l,h)}$ is layer- and head-specific and modulates keys/values, while the regime gate $\rho_t^{(h)}$ is head-specific, shared across layers, and modulates queries through the smoothed utility states. In practice, the readouts are linear projections with tanh activation~\citep{lecun1998efficient}:
\begin{equation}
  g_u^{(l,h)}(u_j) = \tanh(W_s^{(l,h)}\, u_j), \qquad
  f_u^{(h)}(\bar{u}_t) = \tanh(W_\rho^{(h)}\, \bar{u}_t),
  \label{eq:gating-readout}
\end{equation}
where $W_s^{(l,h)}, W_\rho^{(h)} \in \mathbb{R}^{d_m \times d_u}$ are learnable embedding matrices initialized to zero. 

\subsection{Feedback Conditioned Attention}\label{sec:utility-affine}

For the current embedded observation $x_t$ and each historical embedded observation $x_j (j\le t)$, we replace the standard attention projections in equation~\eqref{eq:qkv} with feedback-conditioned attention:
\begin{align}
  \tilde{q}_{t}^{(l,h)}
  &= \Bigl(W_{Q,h}^{(l)} + \operatorname{diag}\!\bigl(\rho_{t}^{(h)}\bigr)\, W_{Q,h}^{(u,l)}\Bigr)\, x_{t}
     + U_{Q,h}^{(l)}\, \rho_{t}^{(h)},
  \label{eq:q-prime} \\
  \tilde{k}_{j}^{(l,h)}
  &= \Bigl(W_{K,h}^{(l)} + \operatorname{diag}\!\bigl(g_{j}^{(l,h)}\bigr)\, W_{K,h}^{(u,l)}\Bigr)\, x_{j}
     + U_{K,h}^{(l)}\, g_{j}^{(l,h)},
  \label{eq:k-prime} \\
  \tilde{v}_{j}^{(l,h)}
  &= \Bigl(W_{V,h}^{(l)} + \operatorname{diag}\!\bigl(g_{j}^{(l,h)}\bigr)\, W_{V,h}^{(u,l)}\Bigr)\, x_{j}
     + U_{V,h}^{(l)}\, g_{j}^{(l,h)}.
  \label{eq:v-prime}
\end{align}
The \emph{utility-affine matrices} $W_{Q}^{(u)}, W_{K}^{(u)}, W_{V}^{(u)}$ share the shape of the projections $W_{Q}, W_{K}, W_{V} \in \mathbb{R}^{d_m \times d}$, preserving the structure while introducing a feedback-conditioned modulation pathway. $\operatorname{diag}(\cdot)$ turns utility-gating vectors $\rho_t, g_j$ into a diagonal matrix suited to the dimensions of utility-affine matrices. The bias matrices $U_{Q}, U_{K}, U_{V} \in \mathbb{R}^{d_m \times d_m}$ inject a utility-dependent offset that remains active even when observations have small magnitude. All learnable matrices are head- and layer-specific. Then, the resulting single-head outputs are aggregated by multi-head attention and then processed by the usual residual and FFNs~\citep{vaswani2017attention}, yielding the layer representation $h_t^{(l)} \in \mathbb{R}^{d}$ passed to the next layer. The empirical contribution of each component in $\tilde{q}_t^{\top} \tilde{k}_j$ is examined in Appendix~\ref{app:logit-decomposition}.

After $L$ layers, the final-layer representation $h_t^{(L)} \in \mathbb{R}^{d}$ is mapped to an action distribution:
\begin{equation}
  \pi_\theta(\cdot \mid \mathcal{H}_t)
  = \mathrm{softmax}\!\bigl(W_{\mathrm{head}} h_t^{(L)} + b_{\mathrm{head}}\bigr).
  \label{eq:task-head}
\end{equation}
For discrete action spaces, $\pi_\theta$ defines a categorical policy over $\mathcal{A}$ and the action is selected greedily as $ a_t=\arg\max_{a\in\mathcal A}\pi_\theta(a\mid \mathcal{H}_t)$. For continuous allocation tasks (e.g., portfolio management), the softmax output is directly interpreted as a target weight vector over assets.

\textbf{Zero Degradation.} This design admits a zero-degradation property. When the utility gates vanish: $g_j^{(l,h)}=\rho_t^{(h)}=\mathbf{0}$, all utility-conditioned terms in Eqs.~\eqref{eq:q-prime}--\eqref{eq:v-prime} disappear, so UAT recovers Vanilla Transformer. Thus, uninformative feedback can be ignored without altering the observation-only backbone. This recovery point also anchors the theoretical comparison in section~\ref{sec:theory}.


\subsection{Training Strategy}
\label{sec:training}


\begin{algorithm}[t]
\caption{UAT training: rollout-then-optimize.}
\label{alg:train}
\begin{algorithmic}[1]
\Require batch size $B$ (episodes), window length $T$, model $\pi_\theta$, expert targets $a^\ast$
\State Sample $B$ episodes; set initial history $\mathcal{H}_{b,1}=\{o_{b,1}\}$ for $b=1,\dots,B$.
\smallskip
\State \textbf{Rollout (sequential, no parameter update).}\enspace For $t=1,\dots,T$:
\Statex\quad sample $\hat a_{b,t}\sim\pi_\theta(\cdot\mid\mathcal{H}_{b,t})$, advance the process, and store $(o_{b,t+1},\,r_{b,t},\,a^\ast_{b,t})$.
\smallskip
\State \textbf{Optimize (single batched forward pass over $(B,T)$).}
\Statex\quad compute utility states and gates $\{u_{b,t},g_{b,t},\rho_{b,t}\}_{t=1}^{T}$ using Eqs.~\eqref{eq:ue-simple}--\eqref{eq:discount-u}, with $\hat a_{b,0}=r_{b,0}=0$;
\Statex\quad logits $\ell_{b,t}\!=\!\mathrm{UAT}_\theta(o_{b,1:T},\,g_{b,1:T},\,\rho_{b,1:T})$ under a causal mask;
\Statex\quad update $\theta$ by Adam $\mathcal{L}\!=\!-\tfrac{1}{BT}\sum_{b,t}\log p_\theta(a^\ast_{b,t}\!\mid\!\ell_{b,t})$.
\smallskip
\State \textbf{Special case.}\enspace If actions and rewards are dataset-supplied, the rollout step is skipped.
\end{algorithmic}
\end{algorithm}

\textbf{Initialization.} We initialize the bias-free gating readouts to zero, $W_s^{(l,h)}, W_\rho^{(h)}=\mathbf{0}$, which yields $g_j^{(l,h)}=\rho_t^{(h)}=\mathbf{0}$ at initialization. The utility-affine matrices are initialized from $\mathcal{N}(0,0.1^2)$ to preserve gradient flow, while the utility-bias matrices are initialized to zero; all utility-conditioned terms are initially gated out. Hence UAT starts from the Vanilla Transformer and gradually activates the feedback modulation pathway during training.

\textbf{Expert Label Loss.} All models are trained against domain expert action labels, falling within a behavioral cloning paradigm~\citep{bain1999framework, pomerleau1989alvinn}. By fixing the training objective to supervised imitation, performance can be attributed solely to architecture capacity rather than variance in the optimization~\citep{ross2011dagger}. For discrete actions, all models minimize the cross-entropy loss~\citep{shannon1948mathematical} against the expert action $a_t^*$:
\begin{equation}
  \mathcal{L}_{\mathrm{CE}}
  = -\frac{1}{BT}\sum_{b=1}^{B}\sum_{t=1}^{T}
    \log \pi_\theta\!\bigl(a_{b,t}^* \mid \mathcal{H}_{b,t}\bigr),
  \label{eq:ce-loss}
\end{equation}
where $B$ is the batch size and $T$ is the episode length. For continuous allocation (Domain~C), the objective is replaced by a KL divergence~\citep{cover2006elements} against hindsight-optimal soft labels~\citep{hinton2015distilling}. Domain-specific expert construction details are provided in Appendix~\ref{app:expert-details}.

\textbf{Rollout-then-optimize.} UAT trains with rollout-then-optimize strategy~\citep{sutton2018reinforcement}: it first rolls out causally with no gradients to collect $\{(a_t, r_t)\}_{t=1}^{T}$, akin to on-policy RL~\citep{schulman2017ppo}; the shifted pairs $\{(a_{t-1}, r_{t-1})\}$ are then fed to the utility encoder, and a causally-masked parallel forward pass over $(B, T)$ tensor yields all logits, on which~\eqref{eq:ce-loss} is backpropagated each step. When actions and rewards are dataset-supplied (Domain~D), UAT skips the rollout process. At test, UAT performs causal step-by-step inference without any expert guidance. Algorithm~\ref{alg:train} summarizes the procedure.

\subsection{Parameter and Computational Overhead}
\label{sec:overhead}

UAT introduces only a small forward-pass overhead over its Transformer counterparts during both training and inference. Detailed analysis is provided in Appendix~\ref{app:eff}.

\section{Theoretical Analysis}
\label{sec:theory}



This section formalizes the theoretical properties of UAT. We first state assumptions explicitly, see details in Appendix~\ref{prof:assum}, and then establish the strict representational separation property, and uniform approximation result for feedback-dependent decision maps.


Let $\mathcal{F}_{\mathrm{TF}}^{\mathrm{obs}}$ denote the class of functions representable by observation-only Transformers. Let $\mathcal F_{\mathrm{UAT}}$ denote the class of functions representable by UAT with action--reward feedback.

\begin{theorem}[Strict representational hierarchy]
\label{thm:strict-hierarchy} For finite horizon $T\ge 2$, if the feedback space contains at least two distinguishable feedback outcomes, then
\begin{equation}
 \mathcal{F}_{\mathrm{TF}}^{\mathrm{obs}} \subsetneq \mathcal{F}_{\mathrm{UAT}}.
 \end{equation}
\end{theorem}

Theorem~\ref{thm:strict-hierarchy} states a \emph{structural separation}: there exist feedback-dependent mappings that observation-only retrieval can never represent, regardless of optimization.



\begin{theorem}[Universal approximation for feedback-dependent decision maps]
\label{thm:approximation}
Under Assumptions~\ref{ass:compact}--\ref{ass:lipschitz} in Appendix~\ref{prof:assum}, for every $\varepsilon > 0$ and every Lipschitz decision map $F^*: \mathcal{H}_t \to \Delta(\mathcal{A})$, there exists a finite-depth, finite-width UAT such that
\begin{equation}
\sup_{H_t \in \mathcal{H}_t} \|F_{\UAT}(H_t) - F^*(H_t)\| < \varepsilon.
\label{eq:uniform-approx}
\end{equation}
\end{theorem}


Theorem~\ref{thm:approximation} is phrased as \emph{universality inheritance} result. This is more rigorous than claiming unconditional universality directly from the UAT equations. It isolates the contribution of the feedback pathway while transparently stating all assumptions borrowed from the underlying Transformer approximation theory. The full proof of Theorem~\ref{thm:strict-hierarchy} and Theorem~\ref{thm:approximation} is given in Appendix~\ref{app:proof-theo}.

\section{Experiments}
\label{sec:experiments}

Our experimental evaluation is organized into five strands: (i) overall performance; (ii) boundary analyses testing robustness to feedback degradation; (iii) ablations isolating each architectural component;  (iv) mechanism analyses visualizing how attention is reallocated; and (v) efficiency and sensitivity studies. This structure allows us to connect empirical findings to theoretical predictions. All experiments use an RTX A6000 GPU (48GB). 

\subsection{Experimental Setup}

\label{sec:exp-design}


\textbf{Datasets.}
\label{sec:exp-bench}
Consider four benchmarks covering representative settings: \textbf{(A) Synthetic POMDP (Dark Room)} with episode-wise hidden goals under three across-episode schedules~\citep{zintgraf2020varibad,laskin2023ad}, \textbf{(B) non-stationary clinical treatment}~\citep{johnson2016mimic} with regime-switching treatment efficacy, \textbf{(C) financial trading} with market regime transitions~\citep{yahoo2025finance}, and \textbf{(D) movie recommendation}~\citep{harper2015movielens} with delayed genre-preference signals. Detailed settings including dataset statistics, features, hyperparameters and evaluation metrics are provided in Appendix~\ref{appendix:exp}.

\textbf{Baselines.}
\label{sec:baselines}
We compare UAT against the following baselines. \textbf{(1) Cross-domain baselines:} \emph{observation-only:} VT~\citep{vaswani2017attention} and GTrXL~\citep{parisotto2020stabilizing}; \emph{test-time adaptation:} TENT~\citep{wang2021tent} and CoTTA~\citep{wang2022cotta}; \emph{input-level feedback:} DT-style~\citep{chen2021decision} and A-DT, which concatenates shifted $(a_{t-1},r_{t-1})$ into each observation (defined in Proposition~\ref{prop:adt_limit}). 

\textbf{(2) Domain-specific baselines:} IQL~\citep{kostrikov2022offline}, CaDM~\citep{lee2020cadm} for Domain~B; DeepTrader~\citep{wang2021deeptrader}, SARL~\citep{ye2020} for Domain~C; SASRec~\citep{kang2018sasrec}, BERT4Rec~\citep{sun2019bert4rec} for Domain~D. We also design six ablation variants to test the component-specific contribution detailed in Appendix~\ref{app:ablation}.



\subsection{Overall Performance}
\label{res:overall}

\begin{table*}[t]
\centering\footnotesize
\renewcommand{\arraystretch}{1.05}
\setlength{\tabcolsep}{2.0pt}
\newcommand{\sd}[1]{{\scriptsize\textcolor{gray!70}{$\pm$#1}}}
\newcommand{\val}[2]{#1\,\sd{#2}}
\newcommand{\valb}[2]{\textbf{#1}\,\sd{#2}}
\newcommand{\valu}[2]{\underline{#1}\,\sd{#2}}
\newcommand{\grp}[2]{\multirow{#1}{*}{\parbox{0.68in}{\centering\scriptsize #2}}}
\newcommand{\familyrule}{\addlinespace[1pt]\cmidrule(lr){1-12}\addlinespace[1pt]}
\caption{Main results, 3 seeds, mean$\pm$std. \textbf{A}: acc (Gradual/Abrupt/Cyclical). \textbf{B}: overall / post-shift acc. \textbf{C}: annualized Sharpe (DJIA/HSI/CSI). \textbf{D}: HR@10, NDCG@10. \textbf{Bold}: best; \underline{underline}: second.}
\label{tab:main}
\resizebox{\textwidth}{!}{%
\begin{tabular}{@{}ll ccc cc ccc cc@{}}
\toprule
\textbf{Family} & \textbf{Model}
& \multicolumn{3}{c}{\textbf{A} (Acc)}
& \multicolumn{2}{c}{\textbf{B}}
& \multicolumn{3}{c}{\textbf{C} (SR)}
& \multicolumn{2}{c}{\textbf{D}} \\
\cmidrule(lr){3-5}\cmidrule(lr){6-7}\cmidrule(lr){8-10}\cmidrule(l){11-12}
& & Grad.\,$\uparrow$ & Abr.\,$\uparrow$ & Cyc.\,$\uparrow$
& Acc\,$\uparrow$ & Post\,$\uparrow$
& DJIA\,$\uparrow$ & HSI\,$\uparrow$ & CSI\,$\uparrow$
& HR\,$\uparrow$ & NDCG\,$\uparrow$ \\
\midrule

\grp{2}{Obs-only}
& Vanilla-TF
& \val{0.362}{0.01} & \val{0.356}{0.02} & \val{0.340}{0.01}
& \val{0.700}{0.00} & \val{0.408}{0.02}
& \val{1.036}{0.011} & \val{0.243}{0.026} & \val{0.383}{0.018}
& \val{0.296}{0.005} & \val{0.135}{0.002} \\

& GTrXL
& \val{0.328}{0.00} & \val{0.336}{0.01} & \val{0.366}{0.01}
& \val{0.636}{0.01} & \val{0.413}{0.02}
& \val{1.030}{0.050} & \val{0.290}{0.142} & \valu{0.463}{0.060}
& \val{0.311}{0.004} & \val{0.142}{0.001} \\
\familyrule

\grp{2}{TTA}
& TENT
& \val{0.346}{0.03} & \val{0.354}{0.01} & \val{0.346}{0.01}
& \val{0.675}{0.05} & \valu{0.493}{0.05}
& \val{1.024}{0.096} & \val{0.524}{0.120} & \val{0.350}{0.021}
& \val{0.244}{0.002} & \val{0.113}{0.000} \\

& CoTTA
& \val{0.346}{0.03} & \val{0.354}{0.01} & \val{0.346}{0.01}
& \val{0.675}{0.05} & \valu{0.493}{0.05}
& \val{1.029}{0.031} & \valu{0.528}{0.052} & \val{0.392}{0.047}
& \val{0.244}{0.002} & \val{0.113}{0.000} \\
\familyrule

\grp{2}{Input-fb}
& DT-style
& \val{0.658}{0.21} & \val{0.738}{0.00} & \val{0.791}{0.01}
& \val{0.706}{0.04} & \val{0.442}{0.02}
& \valu{1.052}{0.063} & \val{0.494}{0.128} & \val{0.382}{0.035}
& \valu{0.339}{0.005} & \valb{0.155}{0.003} \\

& A-DT
& \valu{0.866}{0.00} & \valu{0.826}{0.00} & \valu{0.838}{0.00}
& \val{0.667}{0.01} & \val{0.414}{0.03}
& \val{1.049}{0.040} & \val{0.283}{0.054} & \val{0.372}{0.055}
& \val{0.337}{0.003} & \valu{0.154}{0.002} \\
\familyrule

\grp{6}{Domain-spec.}
& IQL$^{\text{B}}$
& \textendash & \textendash & \textendash
& \valu{0.712}{0.02} & \val{0.425}{0.02}
& \textendash & \textendash & \textendash
& \textendash & \textendash \\

& CaDM$^{\text{B}}$
& \textendash & \textendash & \textendash
& \val{0.375}{0.06} & \val{0.253}{0.03}
& \textendash & \textendash & \textendash
& \textendash & \textendash \\

& DeepTrader$^{\text{C}}$
& \textendash & \textendash & \textendash
& \textendash & \textendash
& \val{0.983}{0.012} & \val{-0.019}{0.005} & \val{0.461}{0.010}
& \textendash & \textendash \\

& SARL$^{\text{C}}$
& \textendash & \textendash & \textendash
& \textendash & \textendash
& \val{1.017}{0.007} & \val{0.043}{0.006} & \val{0.460}{0.005}
& \textendash & \textendash \\

& SASRec$^{\text{D}}$
& \textendash & \textendash & \textendash
& \textendash & \textendash
& \textendash & \textendash & \textendash
& \val{0.281}{0.010} & \val{0.129}{0.003} \\

& BERT4Rec$^{\text{D}}$
& \textendash & \textendash & \textendash
& \textendash & \textendash
& \textendash & \textendash & \textendash
& \val{0.160}{0.004} & \val{0.073}{0.002} \\
\midrule

\grp{1}{Ours}
& \textbf{UAT}
& \valb{0.873}{0.00} & \valb{0.827}{0.00} & \valb{0.842}{0.00}
& \valb{0.808}{0.02} & \valb{0.599}{0.01}
& \valb{1.071}{0.038} & \valb{0.554}{0.123} & \valb{0.473}{0.013}
& \valb{0.341}{0.002} & \valb{0.155}{0.001} \\
\bottomrule
\end{tabular}%
}
\end{table*}

\textbf{Strong overall performance of UAT.} Table~\ref{tab:main} shows that UAT achieves the best or statistically comparable results across all benchmarks. Auxiliary metrics are reported in Appendix~\ref{app:diag}, where UAT remains consistently strong. UAT's gains over observation-only baselines indicate that feedback is informative, while its gains over input-level baselines show the gain of conditioning retrieval over just concatenating. The comparison with IQL in Domain B suggests that UAT outperforms reward usage through gradient alone. The weaker results of domain-specific baselines mainly reflect stronger non-stationary settings, which emphasize cross-regime stability rather than overall model quality.

\textbf{The advantage of feedback-conditioned retrieval grows as regime becomes more complex.} In simpler settings like darkroom and recommendation, UAT–A-DT gap is smaller (average accuracy 0.847 (UAT) vs.\ 0.843 (A-DT) in Domain A). In noisier and faster regimes like clinical decision and financial trading, the gap becomes larger (Domain B: 0.599 (UAT) vs.\ 0.442 (DT) / 0.414 (A-DT); Domain C CSI Sharpe: 0.473 (UAT) vs.\ 0.382 (DT) / 0.372 (A-DT)). This suggests that input-level methods have limitations in harder scenarios, where feedback-conditioned retrieval is more effective.

\subsection{Feedback-Quality Degradation Boundary Analysis}
\label{res:sense}

\begin{figure}[t]
\centering
\begin{minipage}[c]{0.7\linewidth}
  \centering
  \includegraphics[width=\linewidth]{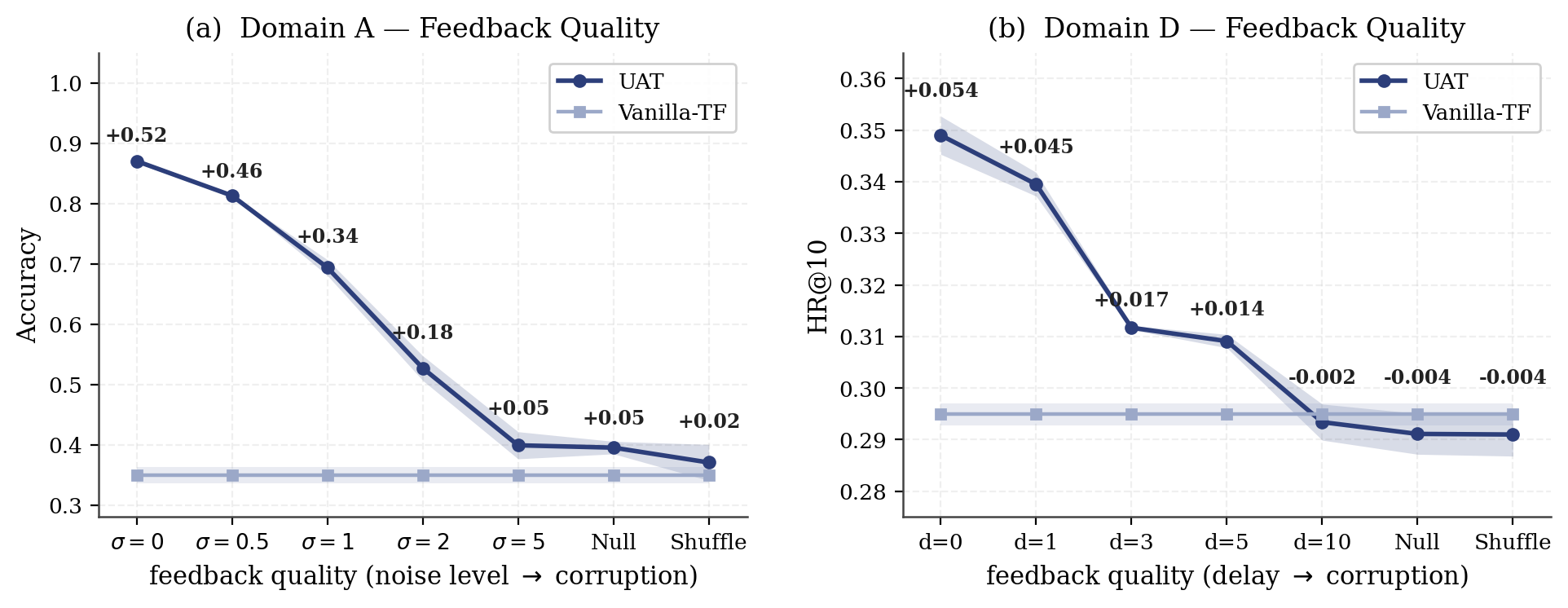}
\end{minipage}\hfill
\begin{minipage}[c]{0.2\linewidth}
  \centering
  \scriptsize
  \setlength{\tabcolsep}{2.5pt}
  \renewcommand{\arraystretch}{1.15}
  \begin{tabular}{@{}lc@{}}
  \toprule
  \multicolumn{2}{c}{\textbf{Inverted}}\\
  \multicolumn{2}{c}{\textbf{reward}}\\
  \midrule
  \textbf{UAT}  & \textbf{VT} \\
  \midrule
  \multicolumn{2}{c}{\emph{Dom.~A (Acc)}}\\
  \textbf{0.870} & 0.349 \\
  \midrule
  \multicolumn{2}{c}{\emph{Dom.~D (HR@10)}}\\
  \textbf{0.341} & 0.296 \\
  \bottomrule
  \end{tabular}
\end{minipage}
\caption{Boundary analysis. Left: Domain~A Gradual under reward noise/corruption; Domain~D under delay/corruption. Bands $=$ std, $\pm$ marks the UAT$-$VT gap. Right: inverted reward ($r{\mapsto}{-}r$).}
\label{fig:boundary}
\end{figure}

We use the Domain~A and Domain~D benchmarks to examine how UAT behaves as the feedback signal becomes progressively less informative. Detailed settings are summarized in Appendix~\ref{app:exp-extra}. 

\textbf{UAT degrades toward the Vanilla Transformer when feedback weakens.}
Figure~\ref{fig:boundary} sweeps feedback signal from clean to corrupted. With informative feedback, UAT leads VT by a large margin (Domain~A: $+0.52$ clean; Domain~D: $+0.054$ at $d{=}0$). The gap shrinks as feedback degrades: on Domain~A, $+0.18$ at $\sigma{=}2$, then $+0.05$ (null) and $+0.02$ (shuffled); on Domain~D, $+0.054$ to $+0.014$ by $d{=}5$, and near VT under null/shuffled. Thus, UAT tracks feedback quality and collapses to VT as signal weakens; the slight Domain~D undershoot is consistent with extra capacity and noise.

\textbf{UAT relies on feedback structure rather than reward polarity.} The table on the right of figure~\ref{fig:boundary} reports inverted-reward setting. Inverting every reward leaves UAT almost unchanged. This indicates that UAT extracts regime from changes and patterns in feedback rather than from absolute value.

\subsection{Ablation Study}
\label{res:ablation}
Full ablation setting and results are reported in Appendix~\ref{app:ablation}. Removing any component weakens UAT, suggesting that the gains come from the integrated feedback-conditioned retrieval design.

\subsection{Mechanism Analysis}
\label{res:attn}

\begin{figure}[t]
    \centering
    \includegraphics[width=0.7\linewidth]{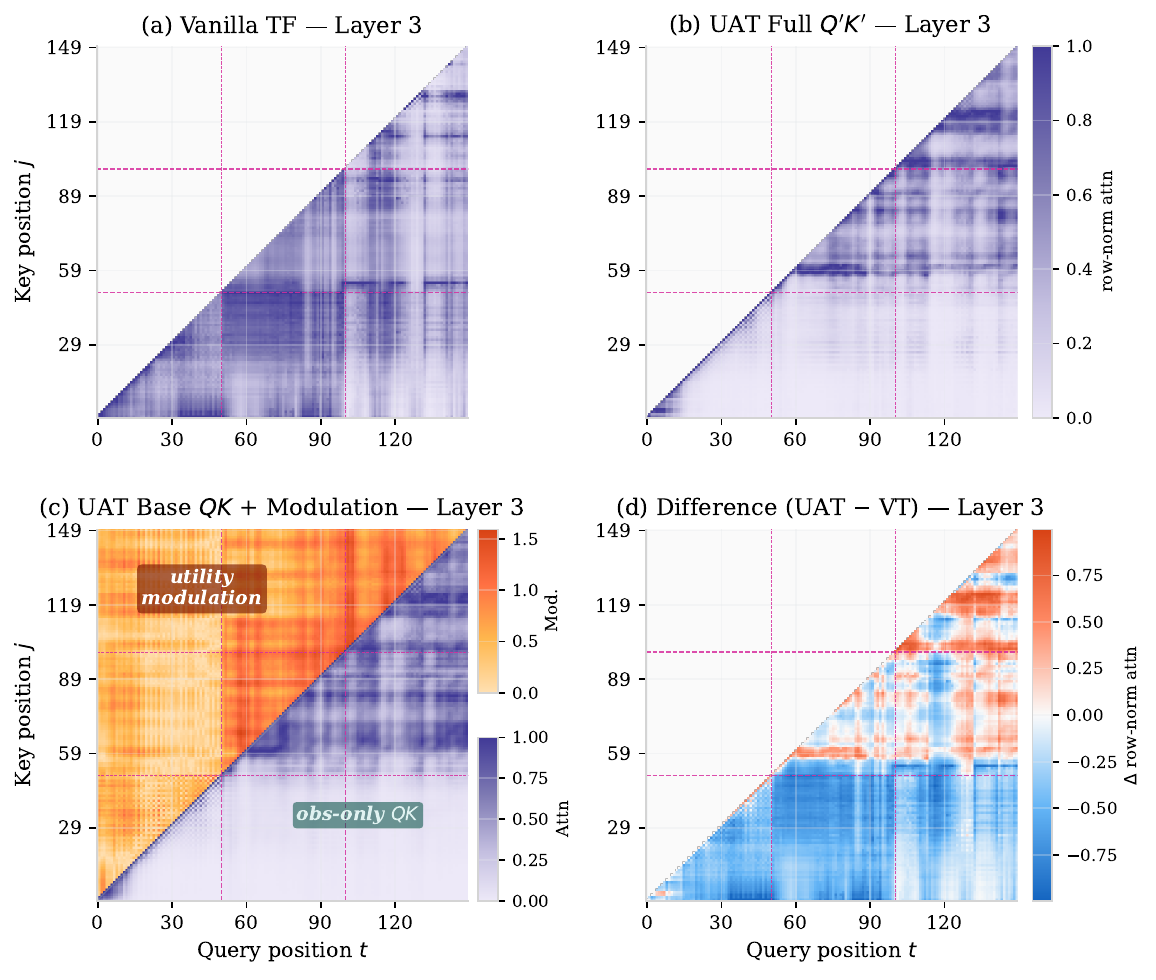}
    \caption{Cyclic-shift attention on DarkRoom
    ($T{=}150$, two shifts; 32-seed mean).
    (a)~VT recalls pre-shift prototypes.
    (b)~UAT refocuses on the active regime.
    (c)~Base attention (lower) with feedback modulation (upper),
    strongest post-shift.
    (d)~UAT--VT logit change.}
    \label{fig:attn_heatmaps}
\end{figure}

\begin{figure}[t]
    \centering
    \includegraphics[width=0.7\linewidth]{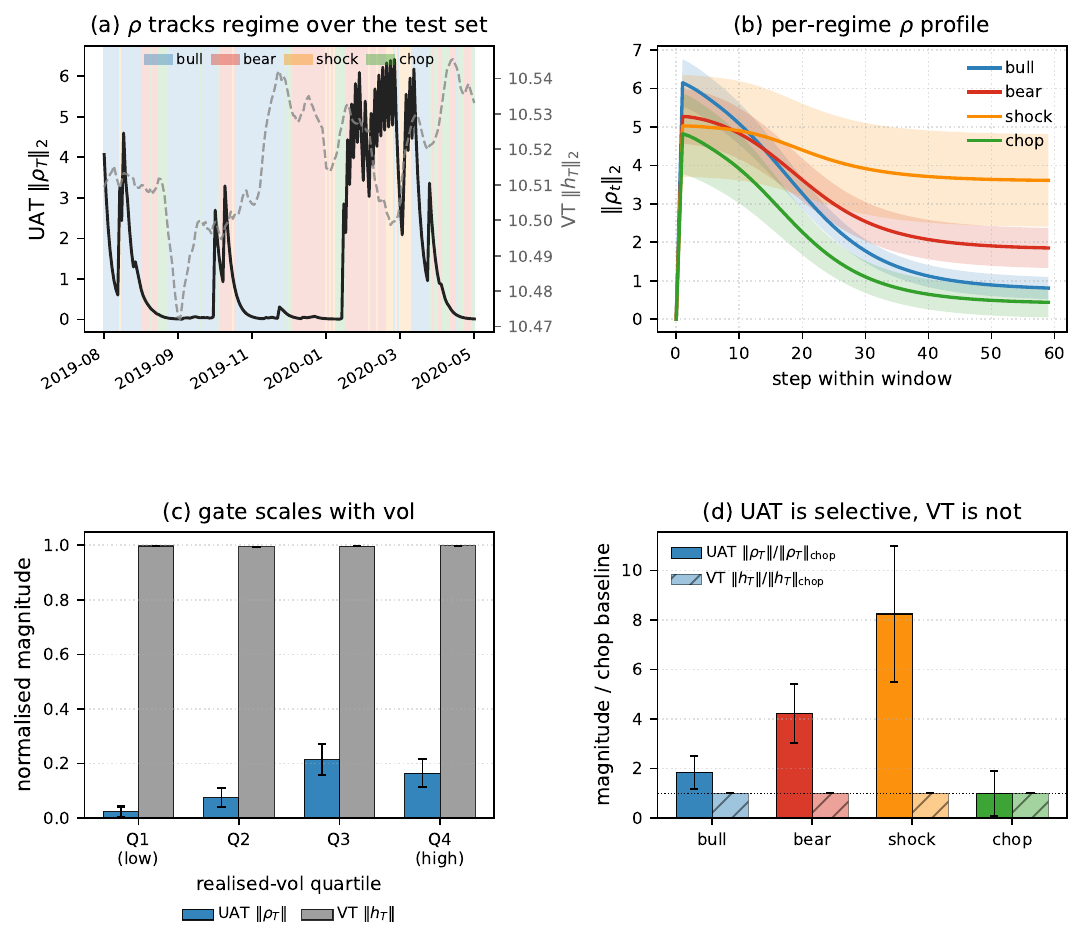}
    \caption{Regime gate on HSI-45.
    (a)~$\|\rho_T\|$ tracks regime changes, while VT
    $\|h_T\|$ remains flat.
    (b)~Within-window $\|\rho_t\|$ separates regimes.
    (c)~Gate magnitude rises with volatility quartile.
    (d)~Shock regimes show the largest increase over the calm baseline.}
    \label{fig:mode_cluster_domC}
\end{figure}

\begin{figure}[t]
    \centering
    \includegraphics[width=0.7\linewidth]{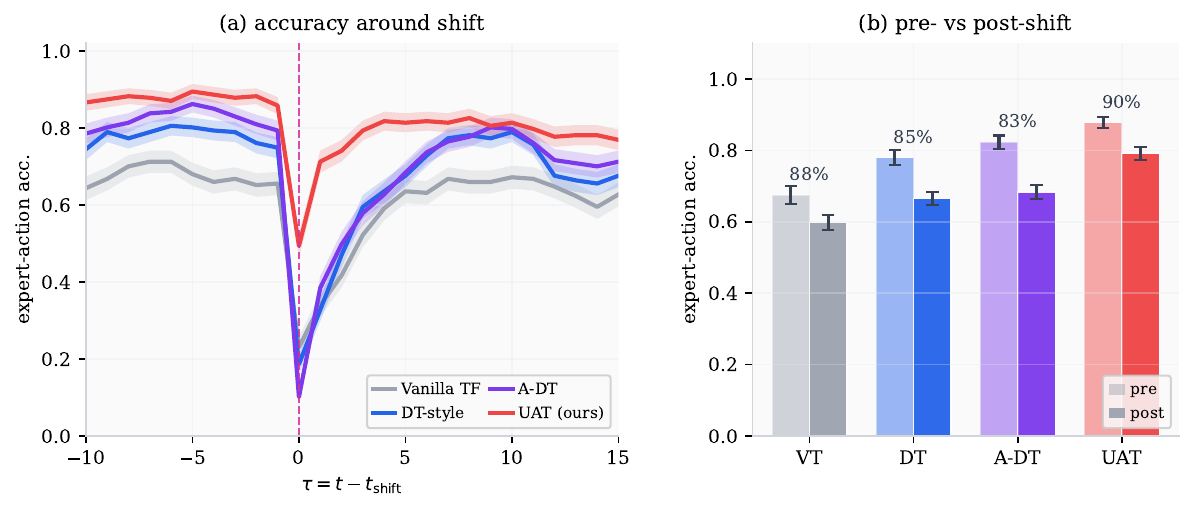}
    \caption{Adaptation kinetics on Domain~B Sepsis
    ($T{=}90$, 64 episodes, 265 shifts).
    (a)~Per-step accuracy around regime shifts ($\tau{=}0$).
    (b)~Mean accuracy before ($[{-}10,{-}1]$) and after
    ($[{+}1,{+}15]$) each shift. Shaded bands denote SEM.}
    \label{fig:kinetics_domB}
\end{figure}

We further analyze how the utility modulation term works. The detailed settings are in Appendix~\ref{app:exp-extra}.

\textbf{UAT relocates attention through modulation.} Figure~\ref{fig:attn_heatmaps} visualizes attention under controlled shifts on Domain~A (150 steps of synthesis). Panel~(a) shows VT's attention concentrates in pre-shift areas, whereas UAT redirects attention to post-shift in panel~(b). Panel~(d) isolates this difference clearly. Panel~(c) further shows that the modulation pathway concentrates after each shift.

\textbf{Gating vectors read regime changes and accelerate adaptation.}
Figure~\ref{fig:mode_cluster_domC} probes UAT's modulation pathway on an HSI-45-only diagnostic during shock periods. Market regimes are assigned post hoc using 21-day forward returns and $z$-scores. The window-end gate $\|\rho_T\|$ rises sharply around 2020 COVID drawdown, while within-window $\|\rho_t\|$ separates regimes in expected stress order: shock $>$ bear $>$ bull $>$ chop. It also increases with realized-volatility quartiles and is largest under shocks relative to calm baseline, indicating that the feedback pathway learns interpretable signals. Figure~\ref{fig:kinetics_domB} further links this mechanism to policy recovery on Sepsis: after a regime shift, UAT recovers to pre-shift accuracy within $\tau{\le}5$, whereas DT-style and A-DT remain $10$--$15$\,pp below their baselines and continue to fluctuate. In total, UAT retains the highest pre-shift accuracy ($90\%$).

\subsection{Efficiency and Sensitivity Study}

\begin{figure}[t]
    \centering
    \includegraphics[width=0.7\linewidth]{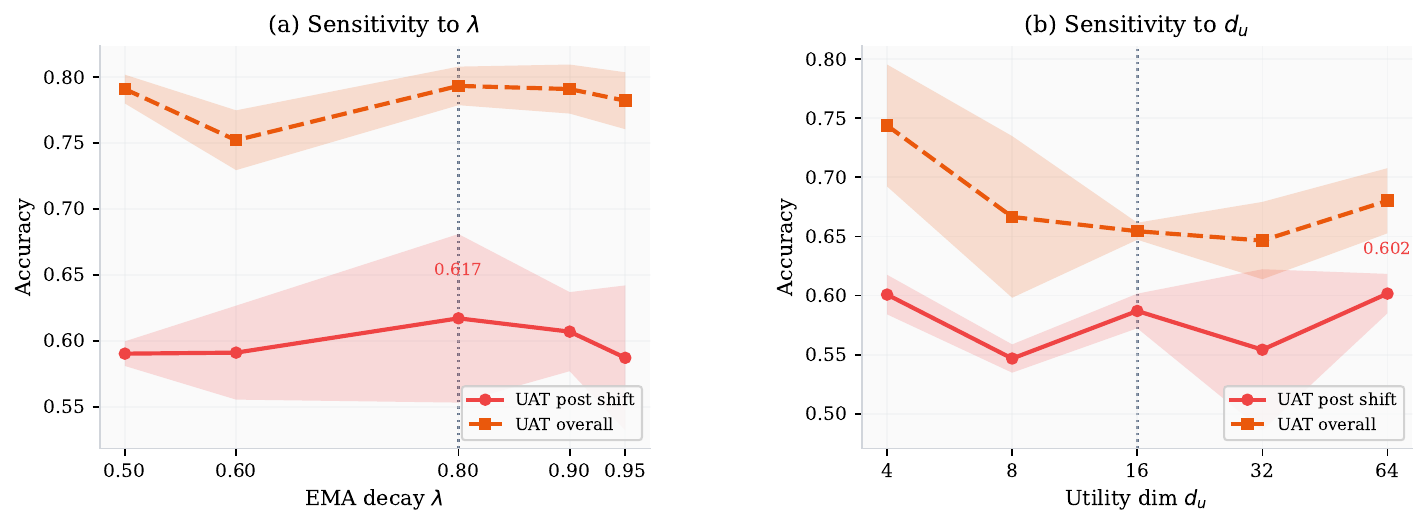}
    \caption{Sensitivity analysis on Domain~B using one-at-a-time
    sweeps over three seeds. Shaded bands denote $\pm$ one standard
    deviation. Solid curves report post-shift expert-action accuracy,
    while dashed curves report overall test accuracy. Vertical lines
    indicate the default settings ($\lambda{=}0.8$, $d_u{=}16$).}
    \label{fig:sensitivity_domB}
\end{figure}

\textbf{Efficiency.} Overall, UAT achieves these gains with moderate computational overhead; details are provided in Appendix~\ref{app:eff}.

\textbf{Sensitivity.} 
UAT-specific and cross-baseline training hyperparameters are summarized in table~\ref{tab:hyperparams} in Appendix~\ref{app:exp-hyp}. Figure~\ref{fig:sensitivity_domB} sweeps the decay $\lambda$ and utility width $d_u$ (one-at-a-time, all else fixed, three seeds). Both are theoretically sensible---$\lambda$ governs utility memory and $d_u$ limits feedback capacity---yet experimental results show that neither drives large swings. For $\lambda$, overall accuracy has cross-seed std mostly $0.01$--$0.02$ (max $0.023$); post-shift std is usually $0.01$--$0.04$. For $d_u$, std on overall accuracy is tight: $\pm 0.007$ at the default $d_u{=}16$; the only outlier is $d_u{=}8$ (low mean, $\pm 0.068$). Defaults $(\lambda{=}0.8,\,d_u{=}16)$ therefore retain strong post-shift performance without brittle tuning.

\section{Conclusion}
\label{sec:conclusion}

We studied sequential decision making under non-stationarity by identifying feedback-blind retrieval as a structural bottleneck of observation-driven Transformer policies. 
After formalizing feedback-informative decision processes, we proposed the Utility-Augmented Transformer (UAT), which conditions query, key, and value projections on a compact action--reward utility state so that feedback can reshape attention retrieval during the forward pass. UAT preserves a zero-degradation path to vanilla attention when feedback is uninformative. Theoretically, we showed that UAT strictly enlarges the observation-only Transformer class and can approximate feedback-dependent decision maps under standard assumptions. Empirically, UAT improves across four non-stationary benchmarks, degrades gracefully under weakened feedback, passes component ablations, exhibits interpretable modulation patterns, and maintains moderate computational overhead.

\clearpage
\bibliographystyle{plainnat}

\bibliography{reference}

\clearpage
\appendix
\section*{Appendix}

\section{Proof Notes}
\label{app:proofs}

\subsection{Proof of Feedback-Blind Propositions} \label{prof:prop}

\begin{proof}[Proof of Proposition~\ref{prop:obs_blind}]\label{prof:prop1}

We proceed by induction over the layers. Let $h_j^{(\ell)}$ denote the representation at position $j$ after layer $\ell$ (with $h_j^{(0)} = x_j$ the input embedding).

\textbf{Base case ($\ell = 0$).} The input embedding is $x_j = \LN(W^E o_j + p_j)$, which depends only on $o_j$ and the positional encoding $p_j$. Since $\mathcal{H}_t$ and $\mathcal{H}'_t$ share the same observation sequence, $x_j(\mathcal{H}_t) = x_j(\mathcal{H}'_t)$ for all $j \leq t$.

\textbf{Inductive step.} Assume $h_j^{(\ell)}(\mathcal{H}_t) = h_j^{(\ell)}(\mathcal{H}'_t)$ for all $j \leq t$. The attention output at layer $\ell + 1$ is
\[
\bar{h}_t^{(\ell+1)} = \sum_{j=1}^{t} \softmax_j\left(\frac{\langle W_{Q}^{(\ell)} h_t^{(\ell)},\, W_{K}^{(\ell)} h_j^{(\ell)} \rangle}{\sqrt{d_m}}\right) W_{V}^{(\ell)} h_j^{(\ell)}.
\]
By the inductive hypothesis, all query, key, and value terms are identical for both histories. The softmax and summation are deterministic functions of these terms, so $\bar{h}_t^{(\ell+1)}(\mathcal{H}_t) = \bar{h}_t^{(\ell+1)}(\mathcal{H}'_t)$.

The feed-forward network $FFN^{(\ell)}$ and layer normalization $LN$ are deterministic, so $h_t^{(\ell+1)}(\mathcal{H}_t) = h_t^{(\ell+1)}(\mathcal{H}'_t)$.

By induction, after $L$ layers, $h_t^{(L)}(\mathcal{H}_t) = h_t^{(L)}(\mathcal{H}'_t)$. The task head $\pi_\theta(\cdot \mid \mathcal{H}_t) = \softmax(W_{\text{head}} h_t^{(L)} + b_{\text{head}})$ is deterministic, so the final retrieval representation $R_t(\mathcal{H}_t)=h_t^{(L)}(\mathcal{H}_t) = h_t^{(L)}(\mathcal{H}'_t) = R_t(\mathcal{H}'_t)$.
\end{proof}

\begin{proof}[Proof of Proposition~\ref{prop:adt_limit}]\label{prof:prop2}
Write $x_j = W_e^{(o)} o_j + W_e^{(f)} f_j$ from $x_j = W_e [o_j; f_j]$ with the column partition $W_e = [W_e^{(o)}, W_e^{(f)}]$. The attention logit (dropping the irrelevant $1/\sqrt{d_m}$ scaling) is
\begin{align*}
\ell_{t,j} &= (W_Q x_t)^{\top} (W_K x_j)
= x_t^{\top} W_Q^{\top} W_K x_j \\
&= \big(o_t^{\top} (W_e^{(o)})^{\top} + f_t^{\top} (W_e^{(f)})^{\top}\big)\, W_Q^{\top} W_K\, \big(W_e^{(o)} o_j + W_e^{(f)} f_j\big),
\end{align*}
which expands into the four terms of \eqref{eq:input-level-decomp} with the stated matrices. Each of the four interaction matrices is of the form $A^{\top} W_Q^{\top} W_K B$ with $A, B \in \{W_e^{(o)}, W_e^{(f)}\}$. Since $W_Q^{\top} W_K \in \R^{d \times d}$ has rank at most $d_m$ (it is the product of $W_Q^{\top} \in \R^{d \times d_m}$ and $W_K \in \R^{d_m \times d}$), each interaction matrix has rank at most $d_m$, and
\[
\col\big(A^{\top} W_Q^{\top} W_K B\big) \subseteq \col\big(A^{\top} W_Q^{\top}\big), \qquad
\row\big(A^{\top} W_Q^{\top} W_K B\big) \subseteq \row\big(W_K B\big),
\]
which gives the stated inclusions with $A, B$ instantiated per term. Hence all four interaction modes factor through the single pair of maps $x \mapsto W_Q x$ and $x \mapsto W_K x$.
\end{proof}

\subsection{Assumptions for Theoretical Properties}\label{prof:assum}

\begin{assumption}[Finite horizon and compact domains]
\label{ass:compact}
The horizon $T$ is finite. The observation, action, and reward spaces satisfy $\mathcal{O} \subset \mathbb{R}^{d_o}$, $\mathcal{A} \subset \mathbb{R}^{d_a}$, $\mathcal{R} \subset \mathbb{R}$, and each set is compact.
\end{assumption}

\begin{assumption}[Rich utility encoding]
\label{ass:encoder}
The utility encoder $E: \mathbb{R}^{d_a} \times \mathbb{R} \to \mathbb{R}^{d_u}$ and local gate readouts $\{g_u^{(l,h)}\}$ satisfy: for any finite set of distinct action--reward pairs $\{(a_i, r_i)\}_{i=1}^N$, there exist encoder parameters such that the induced gating vectors $\{g_i = g_u^{(l,h)}(E(f_a(a_i), r_i))\}$ are pairwise distinct for each head $h$ and layer $l$.
\end{assumption}

\begin{assumption}[Universality of Transformers on augmented tokens]
\label{ass:transformer-universal}
Let $\tilde{x}_j = [x_j; g_j] \in \mathbb{R}^{d + d_m}$. The chosen family of standard Transformers is universal on compact finite-length sequences of augmented tokens $\tilde{x}_{1:t}$: for every continuous target map defined on that sequence domain and every $\varepsilon > 0$, there exists a finite-width, finite-depth causal Transformer that approximates it uniformly to within $\varepsilon$.
\end{assumption}

\begin{assumption}[Target regularity]
\label{ass:lipschitz}
The target decision map $F^*: \mathcal{H}_t \to \Delta(\mathcal{A})$ is Lipschitz on the compact history domain induced by Assumption~\ref{ass:compact}. That is, there exists $L_{F^*} > 0$ such that for all $H_t, H'_t \in \mathcal{H}_t$,
\begin{equation}
\|F^*(H_t) - F^*(H'_t)\| \leq L_{F^*} \cdot d_{\mathcal{H}}(H_t, H'_t),
\end{equation}
where $d_{\mathcal{H}}$ is a suitable metric on the history space (e.g., concatenated Euclidean distance on observations, actions, and rewards).
\end{assumption}

\subsection{Proof of Hierarchy and Expressiveness Theorems}
\label{app:proof-theo}


\begin{proof}[Proof of Theorem~\ref{thm:strict-hierarchy}]

\textbf{Part 1: $\mathcal{F}_{\TF}^{\obs} \subseteq \mathcal{F}_{\UAT}$.}

For any Vanilla Transformer $\TF_\theta \in \mathcal{F}_{\TF}^{\obs}$ with parameters $\theta = (W_{Q,h}^{(l)}, W_{K,h}^{(l)}, W_{V,h}^{(l)})_{l,h}$, we construct a corresponding $\UAT_{\theta'}$ with the same base projection matrices and zero utility-affine terms. Specifically, we set $W_Q^{(u,l)} = W_K^{(u,l)} = W_V^{(u,l)} = \mathbf{0}$ for all $l$; set $U_{Q,h}^{(l)} = U_{K,h}^{(l)} = U_{V,h}^{(l)} = \mathbf{0}$ for all $l, h$; set $W_s^{(l,h)} = \mathbf{0}$ and $W_\rho^{(h)} = \mathbf{0}$ for all $l, h$, so that $g_j^{(l,h)} = 0$ and $\rho_t^{(h)} = 0$ for all $j, t, l, h$.

Substituting into Eqs.~\eqref{eq:q-prime}--\eqref{eq:v-prime}:
\begin{align*}
\tilde{q}_t^{(l,h)} &= W_{Q,h}^{(l)} x_t + \diag(0) \cdot \mathbf{0} \cdot x_t + \mathbf{0} \cdot 0 = W_{Q,h}^{(l)} x_t, \\
\tilde{k}_j^{(l,h)} &= W_{K,h}^{(l)} x_j + \diag(0) \cdot \mathbf{0} \cdot x_j + \mathbf{0} \cdot 0 = W_{K,h}^{(l)} x_j, \\
\tilde{v}_j^{(l,h)} &= W_{V,h}^{(l)} x_j + \diag(0) \cdot \mathbf{0} \cdot x_j + \mathbf{0} \cdot 0 = W_{V,h}^{(l)} x_j.
\end{align*}
These are exactly the standard attention projections. Therefore $\TF_\theta = \UAT_{\theta'}$, and every function in $\mathcal{F}_{\TF}^{\obs}$ is also in $\mathcal{F}_{\UAT}$.

\textbf{Part 2: $\mathcal{F}_{\TF}^{\obs} \neq \mathcal{F}_{\UAT}$.}

Let $T \geq 2$ and consider a feedback-informative task (Definition~\ref{def:feedback-informative}). There exist histories $H_t, H'_t$ sharing the same observation sequence $\{o_j\}_{j\leq t}$ but differing at some position $j^* < t$ in their action--reward outcomes: $(a_{j^*}, r_{j^*}) \neq (a'_{j^*}, r'_{j^*})$, and the optimal policies differ: $\pi_t^*(\cdot \mid H_t) \neq \pi_t^*(\cdot \mid H'_t)$.

We construct a UAT that distinguishes $H_t$ and $H'_t$:

By Assumption~\ref{ass:encoder}, there exist utility encoder parameters such that $u_{j^*+1} = E(f_a(a_{j^*}), r_{j^*}) \neq E(f_a(a'_{j^*}), r'_{j^*}) = u'_{j^*+1}$. Consequently, the local gates at position $j^*+1$ for layer $l=1$ and head $h=1$ satisfy $g_{j^*+1}^{(1,1)} \neq g'^{(1,1)}_{j^*+1}$. Let $d = g_{j^*+1}^{(1,1)} - g'^{(1,1)}_{j^*+1} \in \mathbb{R}^{d_m}$. Since $d \neq 0$, there exists an index $k \in \{1, \ldots, d_m\}$ such that $d_k \neq 0$.

We now choose the utility-affine matrices to exploit this difference. Select $W_K^{(u,1)}$ such that its $k$-th row $w_{K,k}^{(u,1)\top}$ satisfies $w_{K,k}^{(u,1)\top} x_{j^*+1} \neq 0$. This is feasible because $x_{j^*+1}$ is a fixed non-zero vector (it depends on $o_{j^*+1}$ via the input embedding), and we can freely set $W_K^{(u,1)}$.

Similarly, select $W_Q^{(1)}$ such that its $k$-th row $w_{Q,k}^{(1)\top}$ satisfies $w_{Q,k}^{(1)\top} x_t \neq 0$. This is also feasible because $x_t$ is fixed and non-zero.

For all other rows $m \neq k$ of $W_K^{(u,1)}$, choose them such that $w_{K,m}^{(u,1)\top} x_{j^*+1} = 0$ (e.g., set those rows to zero). Then the attention logit difference at position $(t, j^*+1)$ for head $h=1$ is:
\begin{align*}
\Delta \ell_{t, j^*+1}^{(1)} &= x_t^\top W_Q^{(1)\top} \diag(g_{j^*+1}^{(1,1)}) W_K^{(u,1)} x_{j^*+1} - x_t^\top W_Q^{(1)\top} \diag(g'^{(1,1)}_{j^*+1}) W_K^{(u,1)} x_{j^*+1} \\
&= x_t^\top W_Q^{(1)\top} \diag(d) W_K^{(u,1)} x_{j^*+1} \\
&= \sum_{m=1}^{d_m} d_m \cdot (w_{Q,m}^{(1)\top} x_t) \cdot (w_{K,m}^{(u,1)\top} x_{j^*+1}) \\
&= d_k \cdot (w_{Q,k}^{(1)\top} x_t) \cdot (w_{K,k}^{(u,1)\top} x_{j^*+1}) \neq 0.
\end{align*}
The last equality holds because $d_m \cdot (w_{K,m}^{(u,1)\top} x_{j^*+1}) = 0$ for all $m \neq k$ (by construction), and the $k$-th term is non-zero (by Steps 2--3).

Since the softmax function is strictly monotone in each logit (holding others fixed), a non-zero logit difference implies a non-zero difference in the attention weight $\alpha_{t, j^*+1}$. Consequently, the attention output $\bar{h}_t^{(1,1)}$ differs between $H_t$ and $H'_t$. This difference propagates through the residual connection, FFN, and subsequent layers. By induction over layers, the final representation $h_t^{(L)}(H_t) \neq h_t^{(L)}(H'_t)$.

Finally, since the action space $|\mathcal{A}| \geq 2$ and the optimal policies differ at $H_t$ versus $H'_t$, we can choose the task head $(W_{\text{head}}, b_{\text{head}})$ such that $\pi_\theta(\cdot \mid H_t) \neq \pi_\theta(\cdot \mid H'_t)$. For instance, if $h_t^{(L)}(H_t) \neq h_t^{(L)}(H'_t)$, then for any action $a$ where $\pi_t^*(a \mid H_t) > \pi_t^*(a \mid H'_t)$, set $W_{\text{head}}$ to amplify the component corresponding to $a$ sufficiently to maintain the ordering.

Thus, this UAT realizes a mapping that distinguishes $H_t$ and $H'_t$, which no observation-only Transformer can do (Proposition~\ref{prop:obs_blind}). Hence $\UAT \notin \mathcal{F}_{\TF}^{\obs}$, proving $\mathcal{F}_{\TF}^{\obs} \neq \mathcal{F}_{\UAT}$.
\end{proof}

\begin{proof}[Proof of Theorem~\ref{thm:approximation}]

We construct the approximating UAT through five steps.

\textbf{Step 1: Decomposition of $F^*$.}
Write $F^*(H_t) = F^*(\{o_j\}_{j\leq t}, \{a_j, r_j\}_{j<t})$. Since $F^*$ is Lipschitz on a compact domain (Assumption~\ref{ass:lipschitz}), it is uniformly continuous (Heine-Cantor theorem). Therefore, for any $\varepsilon > 0$, there exists $\delta > 0$ such that if two histories are $\delta$-close in the metric $d_{\mathcal{H}}$, their images under $F^*$ are $(\varepsilon/3)$-close.

\textbf{Step 2: Utility encoder approximation.}
The utility encoder $E$ is a two-layer MLP. By the universal approximation theorem for MLPs~\cite{hornik1989multilayer}, for any continuous target map $\phi^*: \mathcal{A} \times \mathcal{R} \to \mathbb{R}^{d_u}$ and any $\delta' > 0$, there exists $E$ such that
\begin{equation}
\|E(f_a(a_j), r_j) - \phi^*(a_j, r_j)\| < \delta', \quad \forall (a_j, r_j) \in \mathcal{A} \times \mathcal{R}.
\label{eq:encoder-approx}
\end{equation}

The gate readouts $g_u^{(l,h)}(u) = \tanh(W_s^{(l,h)} u)$ and $f_u^{(h)}(\bar{u}) = \tanh(W_\rho^{(h)} \bar{u})$ are continuous and map onto $(-1, 1)^{d_m}$. For any target gating values $(g_j^*, \rho_t^*) \in [-1, 1]^{d_m} \times [-1, 1]^{d_m}$, there exist encoder and readout parameters achieving $\|g_j - g_j^*\| < \delta'$ and $\|\rho_t - \rho_t^*\| < \delta'$ for all $j \leq t$.

\textbf{Step 3: Fixed-gating equivalence to standard Transformer.}
For any fixed gating configuration $\mathbf{g} = (\{g_j\}_{j\leq t}, \rho_t) \in [-1, 1]^{d_m \times t} \times [-1, 1]^{d_m}$, the UAT attention projections in Eqs.~\eqref{eq:q-prime}--\eqref{eq:v-prime} become:
\begin{align*}
\tilde{q}_t &= \underbrace{\big(W_Q + \diag(\rho_t) W_Q^{(u)}\big)}_{\tilde{W}_Q(\rho_t)} x_t + \underbrace{U_Q \rho_t}_{b_Q(\rho_t)}, \\
\tilde{k}_j &= \underbrace{\big(W_K + \diag(g_j) W_K^{(u)}\big)}_{\tilde{W}_K(g_j)} x_j + \underbrace{U_K g_j}_{b_K(g_j)}, \\
\tilde{v}_j &= \underbrace{\big(W_V + \diag(g_j) W_V^{(u)}\big)}_{\tilde{W}_V(g_j)} x_j + \underbrace{U_V g_j}_{b_V(g_j)}.
\end{align*}
For fixed $\mathbf{g}$, the matrices $\tilde{W}_Q(\rho_t), \tilde{W}_K(g_j), \tilde{W}_V(g_j)$ and biases $b_Q(\rho_t), b_K(g_j), b_V(g_j)$ are constants. Therefore, the UAT with fixed gating is \textbf{exactly equivalent} to a standard Transformer operating on augmented token representations. By Assumption~\ref{ass:transformer-universal}, for any $\varepsilon' > 0$, there exists a finite-depth, finite-width standard Transformer that approximates $F^*$ uniformly to within $\varepsilon'$ on the domain of augmented tokens.

\textbf{Step 4: Gating-space covering argument.}
Let $\mathcal{G} = [-1, 1]^{d_m \times t} \times [-1, 1]^{d_m}$ be the space of all possible gating configurations. Since $\mathcal{G}$ is a compact subset of a finite-dimensional Euclidean space, it is totally bounded: for any $\delta'' > 0$, there exists a finite $\delta''$-net $\{\mathbf{g}^{(i)}\}_{i=1}^N \subset \mathcal{G}$ such that every $\mathbf{g} \in \mathcal{G}$ is within $\delta''$ of some net center $\mathbf{g}^{(i)}$.

By the Lipschitz continuity of UAT's composite mapping (Assumption~\ref{ass:lipschitz} and the smoothness of attention and FFN operations), the output difference between using $\mathbf{g}$ and its nearest net center $\mathbf{g}^{(i^*)}$ is bounded by $L_{\text{UAT}} \cdot \delta''$ for some Lipschitz constant $L_{\text{UAT}}$ depending on the network architecture but not on the specific input. Choosing $\delta'' = \varepsilon/(3L_{\text{UAT}})$ ensures this discretization error is at most $\varepsilon/3$.

\textbf{Step 5: Composition and error accounting.}
We combine the errors from Steps 2--4 using the triangle inequality. For any history $H_t$:
\begin{align*}
\|F_{\UAT}(H_t) - F^*(H_t)\| &\leq \underbrace{\|F_{\UAT}(H_t) - F_{\UAT}(H_t; \text{ideal gates})\|}_{\text{(a) encoder error}} \\
&\quad + \underbrace{\|F_{\UAT}(H_t; \text{ideal gates}) - F_{\UAT}(H_t; \text{nearest net center})\|}_{\text{(b) discretization error}} \\
&\quad + \underbrace{\|F_{\UAT}(H_t; \text{nearest net center}) - F^*(H_t)\|}_{\text{(c) Transformer approximation error}}.
\end{align*}

Term (a): By Step 2, choose encoder parameters such that $\|g_j - g_j^*\| < \delta'$ and $\|\rho_t - \rho_t^*\| < \delta'$ with $\delta' = \varepsilon/(3L_{\text{UAT}})$. The Lipschitz continuity of UAT ensures term (a) $< \varepsilon/3$.

Term (b): By Step 4, the finite net guarantees term (b) $< \varepsilon/3$.

Term (c): By Step 3 (Assumption~\ref{ass:transformer-universal}), choose depth and width such that the standard Transformer approximates $F^*$ to within $\varepsilon/3$ at each net center. Hence term (c) $< \varepsilon/3$.

Summing the three bounds:
\[
\|F_{\UAT}(H_t) - F^*(H_t)\| < \frac{\varepsilon}{3} + \frac{\varepsilon}{3} + \frac{\varepsilon}{3} = \varepsilon.
\]
Since this holds for all $H_t \in \mathcal{H}_t$, we obtain the uniform bound in Eq.~\eqref{eq:uniform-approx}.
\end{proof}

\begin{corollary}[Existence of optimal UAT for feedback-informative tasks]
\label{cor:existence}
Under Assumptions~\ref{ass:compact}--\ref{ass:lipschitz}, for any feedback-informative task (Definition~\ref{def:feedback-informative}), there exists a UAT that achieves strictly lower imitation loss than any observation-only Transformer.
\end{corollary}

\begin{proof}
Let $F^*$ be the expert policy. By Theorem~\ref{thm:approximation}, there exists a UAT, called $\UAT_1$ such that $\|F_{\UAT_1} - F^*\|_\infty < \varepsilon/2$ for any $\varepsilon > 0$. By Definition~\ref{def:feedback-informative}, there exist histories $H_t, H'_t$ where $F^*(H_t) \neq F^*(H'_t)$ but any observation-only Transformer produces $F_{\TF}(H_t) = F_{\TF}(H'_t)$. Choosing $\varepsilon$ smaller than half the minimum separation $\min_{H_t} \|F^*(H_t) - F_{\TF}(H_t)\|$ over all feedback-informative histories ensures that $\UAT_1$ outperforms every $\TF \in \mathcal{F}_{\TF}^{\obs}$.
\end{proof}

\section{Empirical Logit Decomposition}
\label{app:logit-decomposition}

Recall the utility-affine projections in Eqs.~\eqref{eq:q-prime}--\eqref{eq:k-prime}. Dropping the additive bias channels $U_Q \rho_t,\,U_K g_j$ for clarity, the obs-grounded part of the projections reads
\begin{equation*}
  \tilde{q}_t \;=\; \bigl(W_Q + \operatorname{diag}(\rho_t)\, W_Q^{(u)}\bigr)\, x_t,
  \qquad
  \tilde{k}_j \;=\; \bigl(W_K + \operatorname{diag}(g_j)\, W_K^{(u)}\bigr)\, x_j .
\end{equation*}
Expanding the inner product $\tilde{q}_t^{\top} \tilde{k}_j$ gives four logit terms:
\begin{align}
  \tilde{q}_t^{\top} \tilde{k}_j
  &= \underbrace{x_t^\top W_Q^\top W_K\, x_j}_{\ell^{(0)}:\;\text{obs-only}}
   + \underbrace{x_t^\top W_Q^\top \,\operatorname{diag}(g_j)\, W_K^{(u)}\, x_j}_{\ell^{(1)}:\;\text{obs}\times\text{fb}}
   \notag\\[-2pt]
  &\quad
   + \underbrace{x_t^\top \bigl(\operatorname{diag}(\rho_t)\, W_Q^{(u)}\bigr)^{\!\top} W_K\, x_j}_{\ell^{(2)}:\;\text{fb}\times\text{obs}}
   + \underbrace{x_t^\top \bigl(\operatorname{diag}(\rho_t)\, W_Q^{(u)}\bigr)^{\!\top}\!\!\operatorname{diag}(g_j)\, W_K^{(u)}\, x_j}_{\ell^{(3)}:\;\text{fb}\times\text{fb}} .
  \label{eq:four-terms}
\end{align}
Term $\ell^{(0)}$ is identical to vanilla attention. $\ell^{(1)}$ re-ranks each key by its per-token utility $g_j$, while $\ell^{(2)}$ reshapes the query according to the current regime gate $\rho_t$. The cross term $\ell^{(3)}$ activates only when both gates are simultaneously non-trivial. UAT exactly recovers a Vanilla Transformer when $\rho_t,g_j\to\mathbf{0}$, in which case $\ell^{(1)}=\ell^{(2)}=\ell^{(3)}=0$.

The signed range $[-1,1]$ could be read as a directional coefficient on the learned utility-affine bases. Positive entries reinforce the corresponding rows of $W^{(u)}$, negative entries subtract or reverse them, and entries near zero suppress that modulation channel. Hence the direction of the logit change is jointly determined by the gate signs $(g_j,\rho_t)$ and the learned matrices $W_Q^{(u)},W_K^{(u)}$, which lets UAT either increase or decrease retrieval weights according to feedback.

\section{Supplementary Details of Main Experiment}
\label{appendix:exp}

\subsection{Benchmark Datasets and Environment Details}
\label{app:data-details}

\textbf{Domain A (Synthetic DarkRoom).}
A $10{\times}10$ DarkRoom grid~\citep{zintgraf2020varibad,laskin2023ad} with an invisible goal. Each episode starts from a random non-goal cell.
\emph{Observation} (10-dim): the $3{\times}3$ local wall mask (9 binary entries) plus a noisy binary cue that fires with probability $0.5$ when the agent is within Manhattan distance~\citep{agg2001} 2 of the goal.
\emph{Actions} (5): stay/ up / down / left / right (where $a_t = 0$ stands for stay). \emph{Reward}: $+2$ at the goal, $+0.5$ for reducing Manhattan distance, $-0.5$ for increasing it, and $-0.01$ per step. The hidden goal remains fixed within each episode and follows one of three schedules across episodes: \textit{gradual} drift, \textit{abrupt} jumps every six episodes, or a four-waypoint \textit{cyclic} schedule. Model histories and utility states are reset at episode boundaries, and the current goal is never observed directly.

\textbf{Domain B (Non-Stationary Sepsis Treatment).}
Starting states are 39,232 ICU admission snapshots from MIMIC-III~\citep{johnson2016mimic} via the gym-sepsis simulator~\citep{komorowski2018ai}. We use 20k for training and 5k for testing (IID split over patients).

\emph{State preprocessing.}
Raw states are 46-dimensional clinical vectors (vitals, labs, demographics). We standardize them on the training pool and reduce them to 20 PCA components (about 74\% explained variance).

\emph{Action space.}
The original 24 discretized treatment combinations ($4$ vasopressor levels $\times$ $6$ fluid levels) are mapped to 5 ordered intensity clusters by $k$-means on the treatment-level grid. Clusters are ordered by centroid treatment intensity: cluster 0 corresponds to minimal/no treatment, and cluster 4 to the highest-intensity treatment cluster.

\emph{Reward.}
The step reward is composed of a dense intermediate signal and a sparse terminal outcome. At each non-terminal step, the agent receives $r_t = 0.1\,\Delta\mathrm{SOFA}_t + 0.05\,\Delta\mathrm{Lactate}_t$~\citep{johnson2016mimic}, where $\Delta(\cdot)$ denotes improvement (decrease) relative to the previous step. At episode termination, a bonus of $+1$ is awarded for recovery ($\mathrm{SOFA}{<}0.3$ and $\mathrm{Lactate}{<}0.3$) and $-1$ for death; early termination is stochastic with death probability $p_{\text{die}}$ and recovery probability $p_{\text{rec}}$ that depend on the current SOFA and lactate levels. Treatment resistance is modelled by decaying efficacy after $\ge\!3$ consecutive identical actions.

\emph{Designed Hidden regime shift.}
Each episode contains 4--6 hidden regime segments with randomized lengths (sampled and normalized to the $T{=}90$ horizon), switching among three efficacy regimes:
\textbf{(0) Fluid-responsive}: fluids strongly improve MAP and lactate; vasopressors give mild benefit. Expert action: moderate-to-high fluid (cluster 3).
\textbf{(1) Vasopressor-dependent}: vasopressors strongly improve MAP; fluids are harmful through overload. Expert action: highest-intensity treatment (cluster 4).
\textbf{(2) Conservative}: any aggressive treatment is harmful; only minimal intervention (vaso\,$\le\!1$, fluid\,$\le\!1$) helps. Expert action: minimal treatment (cluster 0).
Regime identity is never revealed to the agent; it can only be inferred from how the physiological state responds to treatment. The expert observes the current regime label directly and selects the intensity level that the regime favours.

\textbf{Domain C (Financial Markets).}
Daily individual equities data in three stock index markets; see Table~\ref{tab:domC-markets}.

\emph{Data source.}
DJIA and HSI constituents are fixed as of April 2026, and their daily OHLCV series are obtained from Yahoo Finance~\citep{yahoo2025finance} historical data; CSI\,300 constituents and OHLCV series are collected through AkShare~\citep{akshare2024} using Chinese forward-adjusted (\emph{qfq}) prices. We build features and returns from these stock-level series and allocate weights across them (plus cash), rather than trading the index itself as a single asset. We use the common trading-date intersection within each market, then restrict every constituent series to that shared calendar before computing indicators and building stride-$1$ length-$T$ windows.

\emph{Observation.}
For each asset, the per-step 11-dim feature vector consists of technical indicators computed from OHLCV: RSI, MACD, MACD signal, MACD histogram, ATR, ADX, CCI, MFI, OBV, stochastic \%K, and stochastic \%D~\citep{kara2011}. Features are z-scored per asset using training-period statistics only and clipped to $[-5,5]$. A market with $N$ assets has native observation size $N{\times}11$; the joint model pads all markets to $N_{\max}{=}45$, giving a fixed $45{\times}11{=}495$ input.

\emph{Action.}
A continuous portfolio weight vector over risky assets plus one cash position, output by a softmax. For the shared joint model, the action head has $N_{\max}{+}1{=}46$ dimensions (cash plus 45 padded asset slots). During market-specific rollout and evaluation, the output is sliced to the actual $N{+}1$ active dimensions and re-normalised before computing portfolio returns. In the reported joint-market experiments we set transaction costs to zero for all three markets, matching the market configuration used by the released scripts~\citep{kolm2014portfolio}.

\emph{Training and evaluation protocol (joint cross-market).} Domain~C uses one shared checkpoint per seed. Each market is first split chronologically into an 80\% training segment and a 20\% test segment on its own trading calendar. We then construct stride-1 windows of $T{=}60$ days independently within each split, pad the asset dimension to $N_{\max}$, and concatenate the three markets' training windows into a pooled multi-market training set. Each window is treated as an independent trajectory sample. The shared checkpoint is evaluated separately on each market's held-out windows. Targets use next-day close-to-close price relatives, avoiding same-day leakage from technical indicators. Outputs are padded to $N_{\max}{+}1$ dimensions including cash, then sliced and re-normalised to the active $N{+}1$ assets during evaluation. During evaluation, each rolling window contributes only its final portfolio prediction, yielding one return per trading date for the computation of Sharpe ratio and CAGR.

\begin{table}[h]
\centering\footnotesize
\renewcommand{\arraystretch}{1.05}
\setlength{\tabcolsep}{4pt}
\caption{Domain~C markets. Calendar coverage is through 2026 at data-collection time (see Appendix~\ref{app:data-details} for how experiments form constituent-aligned panels).}
\label{tab:domC-markets}
\begin{tabular}{@{}lrl@{}}
\toprule
\textbf{Market} & \textbf{Assets} & \textbf{Period (calendar)} \\
\midrule
DJIA      & 30 & 2013--2026 \\
HSI       & 45 & 2015--2026 \\
CSI\,300  & 30 & 2018--2026 \\
\bottomrule
\end{tabular}
\end{table}

\textbf{Domain D (Delayed-Feedback Recommendation).}
MovieLens-25M~\citep{harper2015movielens} contains 25M ratings from 162,541 users on 62,423 movies. We filter to users with $\ge\!100$ interactions and items with $\ge\!50$ ratings, then sample up to 20k eligible user sequences in chronological order.

\emph{Genre-regime construction.}
The MovieLens genres are mapped to $K{=}4$ broad preference groups by hand-merging the official pipe-separated genre tags from the dataset. Each user sequence keeps its real chronological item order and is augmented with a hidden regime schedule that switches every 12--20 steps. The regime is exposed only through delayed feedback:
\[
r_t = (\text{rating}_t / 5)\cdot\mathbf{1}[\text{genre}(i_t)=c_t] - 0.2\cdot\mathbf{1}[\text{genre}(i_t)\neq c_t],
\]
where $c_t$ is the active cluster; the main experiment uses delay $d{=}1$. This follows relative typical settings~\citep{chapelle2014modeling, ng1999policy}.

\emph{State, action, and delayed feedback.}
At prediction position $t$, the model observes the logged item prefix through the completed interaction $i_t$ and predicts a regime-aware expert recommendation target $a_t^*$ over the filtered vocabulary; this target is selected from the active genre cluster and need not equal the logged successor. The synthetic reward $r_t$ is computed from the completed logged interaction $(i_t,\mathrm{rating}_t)$ and is therefore dataset-supplied historical feedback rather than feedback for $a_t^*$. In the main setting ($d{=}1$), feedback-aware models receive $\tilde r_t=r_{t-d}$, zero-padded when $t-d<1$. The sequence length is $T{=}50$.

\subsection{Expert Construction and Domain-Specific Objectives}
\label{app:expert-details}

All domains use the behavioral cloning protocol in section~\ref{sec:training}: models observe $\mathcal{H}_t$ and imitate a privileged expert.

\textbf{Domain A (Synthetic DarkRoom).}
The expert observes the true goal and selects the Manhattan-optimal movement. The loss is cross-entropy~\citep{shannon1948mathematical} over 5 actions.

\textbf{Domain B (Non-Stationary Sepsis Treatment).}
The expert observes the active treatment-efficacy regime and selects the intensity level that the regime favors described in Appendix~\ref{app:data-details}. The training objective is cross-entropy~\citep{shannon1948mathematical} over the 5 clustered intensity levels.

\textbf{Domain C (Financial Markets).}
The expert is constructed from hindsight~\citep{hinton2015distilling}: given realized next-period price relatives $y_t \in \mathbb{R}^{N+1}$ including the cash asset ($y_t^{(0)}{=}1$), the target portfolio is a soft label with cash logit $0$ and risky-asset logits $(y_t^{(i)}-1)/\tau$ for $i=1,\ldots,N$; padded asset slots are masked before the softmax. The training objective is the KL divergence~\citep{cover2006elements} between the model output and this hindsight label:
\begin{equation}
  \mathcal{L}_{\mathrm{KL}}
  = \frac{1}{BT}\sum_{b=1}^{B}\sum_{t=1}^{T}
    D_{\mathrm{KL}}\!\bigl(\pi^*_{b,t} \,\|\, \pi_\theta(\cdot \mid \mathcal{H}_{b,t})\bigr).
  \label{eq:kl-loss}
\end{equation}

\textbf{Domain D (Delayed-Feedback Recommendation).}
The expert observes the active hidden genre regime and selects a target item from that regime's genre cluster. The training objective is cross-entropy~\citep{shannon1948mathematical} over the item vocabulary at each sequence position; the delayed reward is an input signal for feedback-aware models, not the supervised target.

\textbf{Feedback access by model family.}
Feedback-blind baselines (Vanilla Transformer, GTrXL, TENT, CoTTA) receive only $\{o_1,\ldots,o_t\}$. DT-style baselines expose shifted reward/return information as input; A-DT also concatenates the shifted action. UAT sends the shifted logged action and the currently available delayed reward to the utility encoder; in Domain~D, the delay is applied to the reward channel, while the supervised target $a_t^*$ is never supplied as feedback.

\subsection{Evaluation Metrics}
\label{sec:metrics}

\textbf{Domain A (Synthetic DarkRoom).}
\emph{Accuracy} against the oracle optimal action measures overall decision quality. \emph{Navigation efficiency} $\eta = d^*/\max(t_{\mathrm{goal}},\, d^*)$, where $d^*$ is the start-to-goal Manhattan distance~\citep{agg2001} and $t_{\mathrm{goal}}$ is the step at which the agent first reaches the goal ($t_{\mathrm{goal}}{=}T$ if never reached), captures how directly the agent exploits feedback to locate the hidden goal; $\eta{=}1$ corresponds to shortest-path navigation. Each shift pattern (gradual, abrupt, cyclic) is evaluated separately.

\textbf{Domain B (Non-Stationary Sepsis Treatment).}
\emph{Accuracy} against the regime-aware oracle on the 5-action clustered space serves as the primary metric, with \emph{per-regime accuracy} across the three treatment efficacy regimes (fluid-responsive, vasopressor-dependent, conservative). \emph{Post-shift accuracy}, computed over the first 5 steps after each regime switch, quantifies adaptation speed.

\textbf{Domain C (Financial markets).}
\emph{Annualized Sharpe ratio}~\citep{led2008} and \emph{annualized return}~\citep{pagan1996} measure risk-adjusted and absolute profitability. Both are computed from the simple daily portfolio returns on each held-out market. Annualized return is reported as CAGR, $\mathrm{wealth}_{T}^{252/T}-1$ for the daily equity-index markets. Transaction costs are zero in the reported joint-market runs~\citep{kolm2014portfolio}.

\textbf{Domain D (Delayed-Feedback Recommendation).}
\emph{Hit Rate@$K$} (HR@$K$) and \emph{NDCG@$K$}~\citep{he2017ncf} ($K{=}10,20$) on held-out user sequences, evaluated under the leave-one-out protocol with 99 sampled negatives. The positive item is the regime-aware expert target; HR@$K$ measures whether it appears in the top-$K$, and NDCG@$K$ additionally rewards higher ranks.

\subsection{Hyperparameters}
\label{app:exp-hyp}

\begin{table}[t]
\centering\footnotesize
\renewcommand{\arraystretch}{1.05}
\setlength{\tabcolsep}{4pt}
\caption{Hyperparameters (shared across main / ablation). ``UE \& regime lr'' applies to encoder $E$ and readout $W_\rho$.}
\label{tab:hyperparams}
\begin{tabular}{@{}llcccc@{}}
\toprule
& \textbf{Parameter} & \textbf{Domain A} & \textbf{Domain B} & \textbf{Domain C} & \textbf{Domain D} \\
\midrule
\multirow{7}{*}{\rotatebox{90}{\scriptsize Shared}}
& $d_{\mathrm{model}}$ / $H$ / $d_{\mathrm{ff}}$ / $L$ & 64/4/128/3 & 64/4/128/3 & 128/4/256/4 & 64/4/128/3 \\
& Batch size $B$ & 256 & 256 & 128 $\times$ 3 = 384 & 256 \\
& Episode / seq length $T$ & 60 & 90 & 60 & 50 \\
& Optimizer & \multicolumn{4}{c}{Adam, grad clip $1.0$} \\
& Weight decay & $10^{-5}$ & $10^{-5}$ & $10^{-5}$ & $10^{-5}$ \\
& Base lr & 2e-3 & 2e-3 & 2e-3 & 1e-3 \\
& Dropout & 0.05 & 0.2 & 0.05 & 0.1 \\
\midrule
\multirow{7}{*}{\rotatebox{90}{\scriptsize UAT}}
& Utility dim $d_u$ & 16 & 16 & 32 & 16 \\
& Action emb dim $d_{\mathrm{act}}$ & 8 & 8 & 64 & 8 \\
& EMA decay $\lambda$ & 0.7 & 0.8 & 0.9 & 0.95 \\
& Gating init $W_s$, $W_\rho$ & \multicolumn{4}{c}{$\mathbf{0}$ (bias-free)} \\
& $W^{(u)}$ init / $U$ init & \multicolumn{4}{c}{$\mathcal{N}(0,0.1^2)$ / $\mathbf{0}$} \\
& UE hidden dim & 32 & 32 & 64 & 32 \\
& UE \& regime lr & 2e-2 & 2e-2 & 2e-2 & 2e-2 \\
\midrule
\multirow{2}{*}{\rotatebox{90}{\scriptsize Train}}
& Max epochs / patience & 500 / 100 & 500 / 100 & 500 / 100 & 500 / 100 \\
& Feedback delay $d$ & — & — & — & 1 \\
\bottomrule
\end{tabular}
\end{table}

Hyperparameter details are in Table~\ref{tab:hyperparams}. Shared architecture and optimizer settings are matched across Transformer-based methods within each domain. UAT-specific parameters were chosen by coarse validation search and reused in ablations unless stated otherwise.

\subsection{Additional Diagnostics for Main Experiment}\label{app:diag}

\begin{figure}[t]
    \centering
    \includegraphics[width=0.92\linewidth]{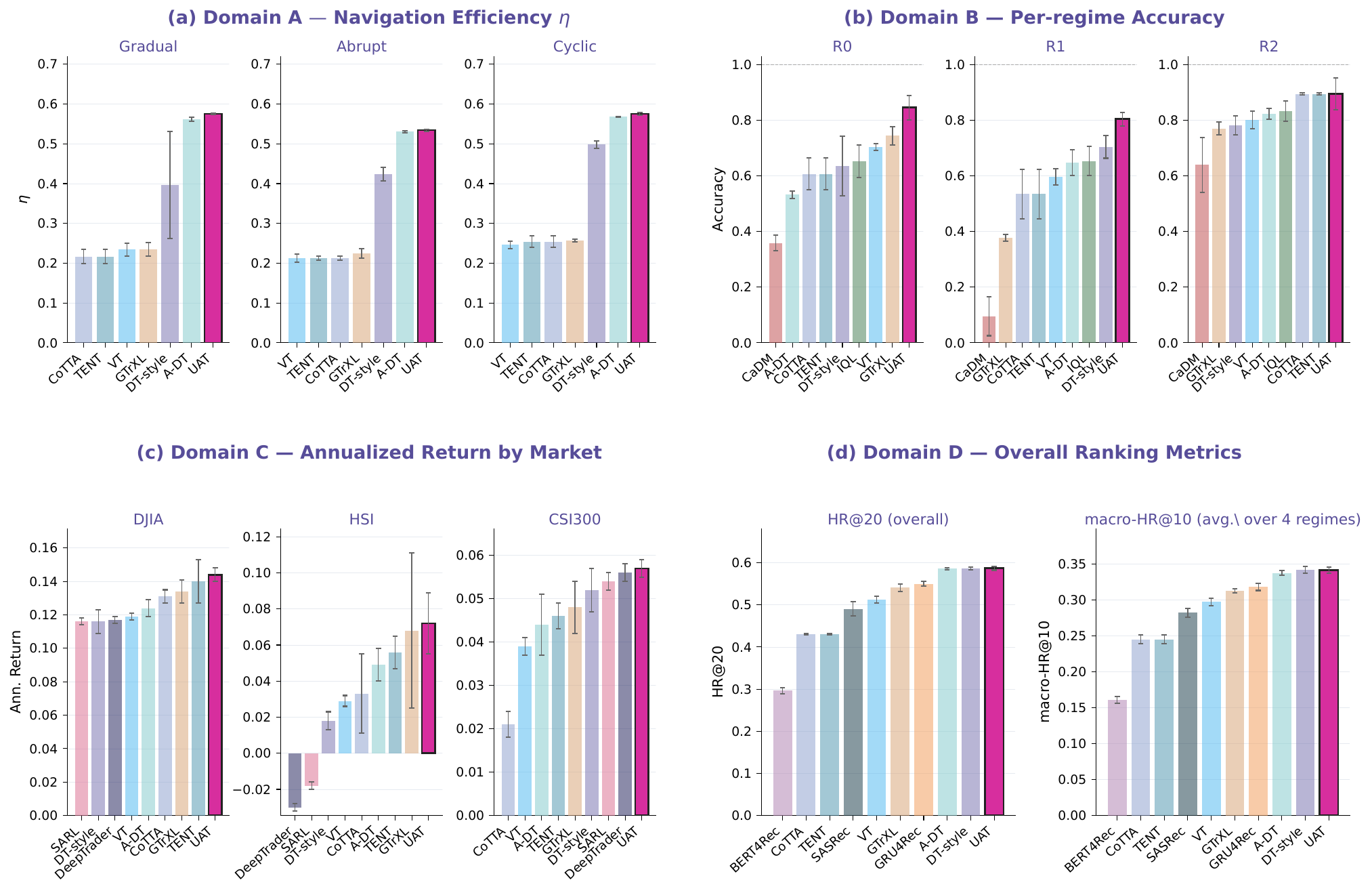}
    \caption{Additional diagnostic metrics. Bars sorted ascending; (a)~Domain~A navigation efficiency $\eta$. (b)~Domain~B per-regime acc. (c)~Domain~C annualized return. (d)~Domain~D HR@20 / HR@10 (over regimes). Error bars $=$ $\pm$std (3 seeds).}
    \label{fig:diag}
\end{figure}

Figure~\ref{fig:diag} reports additional diagnostic metrics not included in the main table. These metrics provide a more fine-grained view across navigation efficiency, regime-wise clinical accuracy, market-level returns, and recommendation ranking performance. 
Across these auxiliary measures, UAT remains consistently among the strongest methods, supporting the robustness of the main results beyond the primary evaluation metrics.

\section{Supplementary Details of Additional Experiments}
\label{app:ar}

\subsection{Ablation Experiment}
\label{app:ablation}

\begin{table}[t]
\centering\footnotesize
\renewcommand{\arraystretch}{1.05}
\setlength{\tabcolsep}{4pt}
\caption{Ablation variants. Same architecture / training as full UAT.}
\label{tab:ablation-design}
\begin{tabular}{@{}lll@{}}
\toprule
\textbf{Variant} & \textbf{Modification} & \textbf{Meaning} \\
\midrule
\multicolumn{3}{@{}l}{\emph{Group I: Utility Encoder}} \\
UAT-noAction  & $a_j{=}0$; reward only          & Isolate reward signal \\
UAT-noReward  & $r_j{=}0$; action only          & Isolate action signal \\
UAT-linear    & $u_j = W[a_j;\, r_j]$; linear   & Remove nonlinear encoding \\
\midrule
\multicolumn{3}{@{}l}{\emph{Group II: Utility-Affine Attention}} \\
UAT-noQ       & $W_Q^{(u)}{=}0$; $\rho_t$ disabled  & Remove query modulation \\
UAT-noKV      & $W_K^{(u)}{=}W_V^{(u)}{=}0$; $g_j$ disabled & Remove key/value modulation \\
UAT-noBias    & $U_Q{=}U_K{=}U_V{=}0$           & Remove observation-independent bias \\
\bottomrule
\end{tabular}
\end{table}

\begin{table*}[t]
\centering\footnotesize
\renewcommand{\arraystretch}{1.05}
\setlength{\tabcolsep}{2.0pt}
\newcommand{\sd}[1]{{\scriptsize\textcolor{gray!70}{$\pm$#1}}}
\newcommand{\val}[2]{#1\,\sd{#2}}
\newcommand{\valb}[2]{\textbf{#1}\,\sd{#2}}
\newcommand{\valu}[2]{\underline{#1}\,\sd{#2}}
\newcommand{\grp}[2]{\multirow{#1}{*}{\parbox{0.68in}{\centering\scriptsize #2}}}
\newcommand{\familyrule}{\addlinespace[1pt]\cmidrule(lr){1-12}\addlinespace[1pt]}
\caption{Ablation, 3 seeds. Columns match Table~\ref{tab:main}. \textbf{Bold}: best; \underline{underline}: second.}
\label{tab:ablation_AC}
\resizebox{\textwidth}{!}{%
\begin{tabular}{@{}ll ccc cc ccc cc@{}}
\toprule
\textbf{Group} & \textbf{Variant}
& \multicolumn{3}{c}{\textbf{A} (Acc)}
& \multicolumn{2}{c}{\textbf{B}}
& \multicolumn{3}{c}{\textbf{C} (SR)}
& \multicolumn{2}{c}{\textbf{D}} \\
\cmidrule(lr){3-5}\cmidrule(lr){6-7}\cmidrule(lr){8-10}\cmidrule(l){11-12}
& & Grad.\,$\uparrow$ & Abr.\,$\uparrow$ & Cyc.\,$\uparrow$
& Acc\,$\uparrow$ & Post\,$\uparrow$
& DJIA\,$\uparrow$ & HSI\,$\uparrow$ & CSI\,$\uparrow$
& HR\,$\uparrow$ & NDCG\,$\uparrow$ \\
\midrule

\grp{3}{Encoder}
& no-action
& \val{0.707}{0.02} & \val{0.605}{0.01} & \val{0.702}{0.02}
& \val{0.793}{0.02} & \val{0.589}{0.02}
& \val{1.062}{0.04} & \val{0.339}{0.10} & \val{0.407}{0.03}
& \valb{0.341}{0.001} & \valu{0.154}{0.001} \\

& no-reward
& \val{0.395}{0.01} & \val{0.415}{0.01} & \val{0.459}{0.02}
& \val{0.706}{0.03} & \val{0.439}{0.04}
& \val{1.045}{0.06} & \val{0.272}{0.07} & \val{0.396}{0.03}
& \val{0.293}{0.002} & \val{0.134}{0.001} \\

& linear
& \val{0.870}{0.00} & \valu{0.826}{0.00} & \val{0.834}{0.00}
& \val{0.740}{0.04} & \val{0.485}{0.03}
& \val{1.038}{0.03} & \val{0.327}{0.15} & \valu{0.417}{0.06}
& \valu{0.340}{0.002} & \valb{0.155}{0.001} \\
\familyrule

\grp{3}{Attention}
& no-Q
& \valu{0.872}{0.00} & \valb{0.827}{0.00} & \val{0.837}{0.01}
& \valu{0.807}{0.01} & \valu{0.591}{0.01}
& \val{1.024}{0.03} & \val{0.251}{0.14} & \val{0.400}{0.05}
& \val{0.339}{0.002} & \valb{0.155}{0.001} \\

& no-KV
& \val{0.506}{0.02} & \val{0.438}{0.00} & \val{0.505}{0.02}
& \val{0.671}{0.03} & \val{0.402}{0.03}
& \val{0.972}{0.02} & \valu{0.487}{0.08} & \val{0.367}{0.03}
& \val{0.294}{0.003} & \val{0.134}{0.001} \\

& no-bias
& \val{0.871}{0.00} & \val{0.826}{0.00} & \valu{0.840}{0.00}
& \val{0.793}{0.03} & \val{0.583}{0.03}
& \valu{1.070}{0.02} & \val{0.462}{0.15} & \valu{0.417}{0.05}
& \val{0.336}{0.001} & \valu{0.154}{0.001} \\
\midrule

\grp{1}{Ours}
& \textbf{UAT}
& \valb{0.873}{0.00} & \valb{0.827}{0.00} & \valb{0.842}{0.00}
& \valb{0.808}{0.02} & \valb{0.599}{0.01}
& \valb{1.071}{0.04} & \valb{0.554}{0.12} & \valb{0.473}{0.01}
& \valb{0.341}{0.002} & \valb{0.155}{0.001} \\
\bottomrule
\end{tabular}%
}
\end{table*}

\textbf{Ablation Variant Design.} Table~\ref{tab:ablation-design} lists the one-factor ablation variants, which could be divided into two groups. \emph{Group~I: Utility encoder} tests the necessity of each input signal and the nonlinear structure: \textbf{no-action} uses reward only; \textbf{no-reward} uses action only; \textbf{linear} replaces the MLP with a linear map. \emph{Group~II: Attention modulation} selectively disables components of the utility-affine attention: \textbf{no-Q} removes $W_Q^{(u)}$; \textbf{no-KV} removes $W_K^{(u)}$ and $W_V^{(u)}$; \textbf{no-bias} removes the utility-bias matrices $U_Q, U_K, U_V$.

\textbf{Design rationale.}
Group~I isolates the feedback inputs and nonlinear utility encoder. Group~II maps to the logit decomposition: noQ removes regime-level query modulation, noKV removes token-level key/value modulation, and noBias removes observation-independent utility offsets. Each variant disables only the targeted component and otherwise preserves the same backbone and all other settings as UAT, ensuring that performance differences isolate the corresponding component.

\textbf{Results.}\label{abr} \emph{(1) All components contribute, while full UAT remains the strongest.} As shown in Table~\ref{tab:ablation_AC}, the full UAT achieves the best overall performance across all domains, indicating that its gains do not come from any single component in isolation. The proposed architecture functions as an integrated system, and removing any one of them weakens the final design. \emph{(2) Component contributions differ.} Among attention-side ablations, removing KV pathway causes the largest degradation on almost all domains: in Domain~A, average accuracy drops from 0.847 to 0.483 in no-KV, much more severely than in no-Q (0.845) or no-bias (0.846). This may suggest that the main gains of attention retrieval come from KV-side, while Q-side and bias mainly provide additional refinement. The encoder-side ablations show that removing reward is more harmful than removing action, matching the intuition that regime shifts are primarily revealed by changed realized reward.

\subsection{Boundary and Mechanism Analyses}
\label{app:exp-extra}

This subsection records protocol differences for boundary and mechanism analyses. The relative results are fully discussed in the main text.

\textbf{Boundary Analysis: Domain A (Reward Quality Degradation).}
UAT and Vanilla-TF are trained/evaluated on gradual DarkRoom under clean rewards, Gaussian reward noise ($\sigma\in\{0.5,1.0,2.0,5.0\}$), null rewards, within-episode reward shuffling, and sign-flipped rewards. Corruption is applied to the feedback channel used by the model; the metric is per-step optimal-action accuracy over three seeds.

\textbf{Boundary Analysis: Domain D (Delay Sensitivity).}
UAT and Vanilla-TF are retrained for $d\in\{0,1,3,5,10\}$ using the same MovieLens genre-regime construction. For this boundary sweep, $d{=}0$ denotes immediate access to $r_t$ from the latest completed logged interaction; it does not expose feedback for the supervised recommendation target $a_t^*$. Other hyperparameters match Table~\ref{tab:hyperparams}; evaluation uses leave-one-out ranking with 99 negatives. The corruption study keeps the main delay $d{=}1$ and applies null/shuffle/adversarial delayed-reward corruption.

\textbf{Mechanism Analysis: Domain A: In-Episode Goal Shift (Left Panel).}
Domain~A checkpoints are rolled out without retraining on an extended cyclic DarkRoom diagnostic (one episode, $T_h{=}150$) with two hidden goal shifts. We average attention maps over 32-seed rollouts and visualize Vanilla-TF attention, UAT attention, the UAT base/modulation decomposition, and the UAT--VT logit difference.

\textbf{Mechanism Analysis: Domain C: Regime Gate on HSI (Right Panel).}
We train and evaluate UAT exclusively on the HSI dataset. Each held-out rolling window is labeled post hoc on its last day into bull, bear, shock, or chop, using information strictly after that day only for the forward-looking statistics (never fed to the model). We form an equal-weight market proxy from constituent simple returns and convert to daily log returns; let $\sigma_{63}$ be the trailing $63$-day standard deviation through the label day. \emph{Shock} takes priority if the label day's absolute log return satisfies $|r_t|/\sigma_{63}>1.5$ \emph{and} the maximum peak-to-trough drawdown over the \emph{subsequent} five trading days exceeds $1.5\%$. If not shock, \emph{bull} applies when the cumulative log return over the next $21$ trading days exceeds $+1\%$, \emph{bear} when it falls below $-1\%$, and \emph{chop} otherwise. We then extract $\|\rho_t\|_2$ at inference time; regime labels are not used for training or adaptation.

\subsection{Efficiency Analysis Experiment}\label{app:eff}

\begin{table}[t]
\centering\footnotesize
\renewcommand{\arraystretch}{1.05}
\setlength{\tabcolsep}{4pt}
\caption{Efficiency measured on Domain~A ($B{=}256$, $T{=}60$, $d_m{=}16$, $H{=}4$, single GPU).}
\label{tab:efficiency}
\begin{tabular}{@{}l r rr rr@{}}
\toprule
& & \multicolumn{2}{c}{\textbf{Training}} & \multicolumn{2}{c}{\textbf{Inference}} \\
\cmidrule(lr){3-4}\cmidrule(lr){5-6}
\textbf{Model} & \textbf{Params} & \textbf{ms/ep} & \textbf{Peak MB} & \textbf{ms/step} & \textbf{Peak MB} \\
\midrule
Vanilla-TF   & 109,828 & 17.1 & 181.4 & 2.1  & 74.9  \\
GTrXL        & 146,884 & 30.2 & 316.2 & 5.3  & 180.4 \\
TENT         & 109,828 & 14.1 & 197.6 & 11.3 & 218.7 \\
CoTTA        & 109,828 & 14.1 & 199.8 & 11.3 & 221.1 \\
DT-style     & 102,660 & 17.4 & 515.4 & 6.2  & 178.7 \\
A-DT         & 102,532 & 14.6 & 202.8 & 2.1  & 94.0  \\
\midrule
\textbf{UAT} & \textbf{155,972} & \textbf{32.4} & \textbf{253.9} & \textbf{3.1} & \textbf{112.0} \\
\bottomrule
\end{tabular}
\end{table}

\textbf{Complexity Analysis.} For hidden size $d$, $H$ heads ($d_m = d/H$), and $L$ layers, \UAT~adds the following parameters, all independent of the horizon~$T$:
\begin{itemize}[leftmargin=1.5em,itemsep=0.2em]
    \item Utility-affine projections: $3LH \cdot d_m \cdot d = 3L d^2$ (since $d_m = d/H$ and $H$ heads);
    \item Utility bias matrices: $3LH \cdot d_m^2 = 3L d_m d$;
    \item Gating readouts: $LH \cdot d_m \cdot d_u + H \cdot d_m \cdot d_u = (L+1)H d_m d_u$;
    \item Utility encoder: $|\phi_u|$ (negligible: two-layer MLP with small hidden width).
\end{itemize}
The total parameter overhead is $O(Ld^2 + Ld_m d + L H d_m d_u) = O(Ld^2)$, since $d_m = d/H$ and $d_u \ll d$ in practice. The per-forward computational cost matches standard attention, $O(BT^2 d)$ for batch size $B$ and horizon $T$, with all \UAT-specific terms being lower-order additions. Training reduces to a single parallel forward pass when actions and rewards are dataset-supplied, and adds an extra length-$T$ rollout phase when rewards depend on the model.

\textbf{Experimental Results.} Table~\ref{tab:efficiency} shows that UAT achieves its performance gains without incurring prohibitive computational cost. UAT's parameter count is comparable to GTrXL (155,972 vs.~146,884), and its training cost is nearly identical (32.4 vs.~30.2 ms/epoch). Its memory usage is substantially lower than DT-style (253.9 vs.~515.4 MB). At inference, UAT is considerably faster than TENT, GTrXL and DT-style.


\end{document}